\documentclass[twoside]{article}

\usepackage[accepted]{aistats2026}
\usepackage[round]{natbib}

\usepackage{amsmath}
\usepackage{amssymb}
\usepackage{amsthm}
\usepackage{mathtools}
\usepackage{graphicx}
\newtheorem{theorem}{Theorem} 
\newtheorem{lemma}[theorem]{Lemma}
\newtheorem{proposition}[theorem]{Proposition}
\newtheorem{assumption}[theorem]{Assumption}
\newtheorem{remark}{Remark} 
\usepackage{braket} 
\usepackage{algorithm}
\usepackage{algorithmic}
\usepackage{hyperref}       
\usepackage{url}    
\usepackage{color}

\newcommand{\R}{\mathbb{R}}
\newcommand{\E}{\mathbb{E}}
\newcommand{\N}{\mathbb{N}}
\newcommand{\T}{^\mathsf{T}}

\newcommand{\e}{\mathrm{e}}
\newcommand*\abs[1]{\left\lvert#1\right\rvert}
\newcommand*\norm[1]{\left\|#1\right\|}

\newcommand{\iid}{\overset{\mathrm{i.i.d.}}{\sim}}

\newcommand{\wh}{\widehat}
\newcommand{\wt}{\widetilde}

\newtheorem{corollary}[theorem]{Corollary}

\newtheorem{notes}{Notes}[section]

\DeclareMathOperator\dif{d\!}

\DeclareMathOperator\var{\mathrm{var}}
\DeclareMathOperator\cov{\mathrm{cov}}
\DeclareMathOperator\Cov{\mathrm{Cov}}
\DeclareMathOperator\MSE{\mathrm{MSE}}

\DeclareMathOperator*{\argmin}{arg\,min}

\begin{document}

%

%
\runningauthor{Ziyang Wei, Wanrong Zhu, Jingyang Lyu, Wei Biao Wu}
\twocolumn[

\aistatstitle{Refining Covariance Matrix Estimation in Stochastic Gradient Descent Through Bias Reduction}

\aistatsauthor{ Ziyang Wei \And Wanrong Zhu \And Jingyang Lyu }

\aistatsaddress{ Department of Statistics\\
  University of Chicago\\
  Chicago, IL 60637, USA \\
  \texttt{ziyangw@uchicago.edu} 
  \And 
  Department of Statistics \\ 
  University of California, Irvine\\ 
  Irvine, CA 92617, USA\\
  \texttt{wanronz1@uci.edu} 
  \And 
  Department of Statistics\\
  University of Wisconsin–Madison \\
  Madison, WI 53706, USA\\
  \texttt{jlyu55@wisc.edu} } 

\aistatsauthor{ Wei Biao Wu }

\aistatsaddress{ Department of Statistics\\
  University of Chicago\\
  Chicago, IL 60637, USA \\
  \texttt{wbwu@uchicago.edu} } 
]

\begin{abstract}

  We study online inference and asymptotic covariance estimation for the stochastic gradient descent (SGD) algorithm. While classical methods—such as plug-in and batch-means estimators—are available, they either require inaccessible second-order (Hessian) information or suffer from slow convergence. To address these challenges, we propose a novel, fully online de-biased covariance estimator that eliminates the need for second-order derivatives while significantly improving estimation accuracy. Our method employs a bias-reduction technique to achieve a convergence rate of $n^{(\alpha-1)/2} \sqrt{\log n}$, outperforming existing Hessian-free alternatives. 
\end{abstract}

\section{INTRODUCTION}
Stochastic gradient descent (SGD), also known as the Robbins–Monro algorithm \citep{robbins1951stochastic}, has become a cornerstone of large-scale machine learning due to its simplicity and notable practical performance \citep{bottou2018optimization}. 
Consider the classic model-parameter estimation setting, where the true model parameter $x^{*}\in \R^{d}$ is characterized as the minimizer of a convex objective function $F(x)$ from $\R^{d}$ to $\R$, i.e., 
\begin{equation}\label{eq:obj}
x^{*} = \underset{x\in \R^{d}}{\arg\min} F(x).
\end{equation}
The objective function $F(x)$ is defined as the expectation of a random loss function $f(x, \xi)$, that is, $F(x) = \E_{\xi \sim \Pi}{f(x, \xi)}$, where $\xi$ is a random variable representing data drawn from the distribution $\Pi$. 
With initial point $x_{0}$, the $i$-th iteration of the SGD algorithm takes the following form
\begin{equation}\label{eq:SGDite}
x_{i} = x_{i-1} - \eta_{i}\nabla{f(x_{i-1}, \xi_{i})},\ i\ge 1,
\end{equation}
where $\{\xi_{i}\}_{i\ge 1}$ is a sequence of \emph{i.i.d} samples from the distribution $\Pi$, $\nabla{f}$ is the gradient of $f(x, \xi)$ with respect to the first argument $x$, and $\eta_{i}$ is the step size at the $i$-th step. 

As the SGD algorithm progresses and converges, its iterates often resemble noisy approximations of the true optimum. Therefore, it is important not only to evaluate convergence but also to assess the reliability of these estimates by understanding the asymptotic distribution and variability of the iterates. In the foundational work of \citet{polyak1992acceleration}, it is shown that, when the step size decays polynomially, i.e., $\eta_{i} = \eta i^{-\alpha}$ with $\eta>0$ and $\alpha\in(0.5, 1)$,  the average of all past SGD iterates, $\bar{x}_{n} = n^{-1}\sum_{i=1}^{n}x_{i}$, exhibits asymptotic normality under suitable conditions:
\begin{equation}\label{eq:asym_norm}
\sqrt{n}(\bar{x}_{n} - x^{*}) \Rightarrow  N(0, \Sigma),
\end{equation}
where $$\Sigma =\nabla^{2}F(x^{*})^{-1}\E\left([\nabla f(x^{*}, \xi)][\nabla f(x^{*}, \xi)]\T \right)\nabla^{2}F(x^{*})^{-1}.$$
If the model is well-specified, this sandwich form limiting covariance matrix achieves the Cramér-Rao lower bound, with $\Sigma^{-1}$ corresponding to the Fisher information matrix, as discussed in \citet{chen2020statistical}. Similar results have been established for alternative variants of SGD—such as those using different weighting strategies or second-order methods—with appropriately adapted limiting covariance matrices \citep{li2022root,na2022asymptotic,wei2023weighted}. 

In the references above, the analytical form of the limiting covariance is studied. However, in practice, the limiting covariance is unknown, as it depends on the underlying data and noise distribution, which are typically not known. Estimating this covariance is therefore essential for quantifying uncertainty in SGD-based estimates and enabling principled statistical inference, such as constructing confidence intervals. Moreover, to stay aligned with the spirit of SGD—namely, computational and memory efficiency, and online updates—it is especially important that the covariance estimation procedure also adheres to these principles. This makes the task of estimating the covariance matrix particularly challenging.

In this paper, we focus on the most fundamental setting: averaged SGD with a polynomially decaying learning rate, $\eta_i = \eta i^{-\alpha}$, where $\eta > 0$ and $\alpha \in (0.5, 1)$. Our goal is to estimate the limiting covariance matrix of $\sqrt{n} \bar{x}_n$, denoted by $\Sigma$ in \eqref{eq:asym_norm}, in a fully online fashion—using only the SGD iterates and without requiring additional computations such as Hessian evaluations. 
Existing methods for online, Hessian-free covariance estimation—which we will discuss later in Section \ref{sec:related}—achieve a best-known convergence rate of $n^{(\alpha-1)/4}$, which is relatively slow compared to the convergence rates of SGD: $\mathcal{O}(1/n)$ for strongly-convex and $\mathcal{O}(1/\sqrt{n})$ for convex problems \citep{nemirovski2009robust, lacoste2012simpler}. A natural question arises: can we improve this result while relying solely on the information provided by a single-pass SGD sequence?

\textbf{Our contributions:}  We give an affirmative response in this paper by presenting the de-biased estimator, a novel covariance estimator refined by the bias-reduction technique. It significantly enhances both the theoretical and practical convergence of the estimation error. Specifically, we show that the error rate of the de-biased estimator is  $n^{(\alpha-1)/2}\sqrt{\log{n}}$, which represents the best known convergence rate to date. Moreover, we propose a single-pass algorithm that updates the estimator using only the SGD iterates, with an update cost of $\mathcal{O} (d^2)$---the minimal computational requirement for estimating a $d\times d$ matrix.

There are alternative inference methods that rely on asymptotic pivotal statistics \citep{lee2022fast, su2023higrad, luo2022covariance, zhu2024high}. Although these methods can sometimes perform better in terms of constructing confidence intervals, they do not yield consistent covariance estimators. Therefore, we do not discuss them in detail here. Beyond statistical inference for stochastic approximation, the estimation of covariance and spectral properties in time series is an independent and well-established topic in the literature \citep{liu2010asymptotics, Flegal2010, xiao2011single, Xiao2012, Chen2013Cov, zhang2017asymptotic}. The novel bias-reduction estimator proposed in our work offers fresh insights into this area, as it can estimate the dependence structure of time-inhomogeneous processes and is amenable to on-the-fly computation.

Throughout the paper, we use the following notation. For a vector $a = (a_1, \ldots, a_d)^\top$, let the norm $\norm{a}_p = (\sum_{i=1}^d a_i^p )^{1/p}$. For a matrix $A \in \R^{d\times d}$, we use $\| A\|_2$ or $\| A\|$ to denote its operator norm and $\| A\|_F$ to denote its Frobenius norm. For $t\in \mathbb{R}$, $\lfloor t \rfloor = \max\{i \in \mathbb Z: i \le t\}$ and $\lceil t \rceil = \min\{i \in \mathbb Z: i \ge t\}$. For positive sequences $\{a_{n}\}$ and $\{b_{n}\}$, $n \in \N$, we write $a_{n}\lesssim b_{n}$ if there exists a positive constant $C$ such that $a_{n} \leq Cb_{n}$ for all $n \in \N$. We write $a_{n}\gtrsim b_{n}$ if $b_{n}\lesssim a_{n}$, and $a_{n}\asymp b_{n}$ if both $a_{n}\lesssim b_{n}$ and $a_{n}\gtrsim b_{n}$. For a finite set $B$, we use $|B|$ to denote its cardinality.

The remainder of the paper is organized as follows. Section \ref{sec:related} reviews existing work on covariance matrix estimation in different settings. In Section \ref{sec:method}, we present the formulation and algorithm for our recursive de-biased estimator.  In Section \ref{sec:theory}, we establish theoretical guarantees for the consistency of our estimator. Section \ref{sec:simulation} demonstrates the superior finite-sample performance of the de-biased method through numerical experiments. Finally, Section \ref{sec:discussion} concludes the paper and outlines directions for future research.


\section{BACKGROUND AND RELATED WORK}\label{sec:related}
We begin by reviewing existing methods for online covariance matrix estimation related to stochastic gradient descent (SGD).
Recall that the limiting covariance matrix $\Sigma$ of the averaged SGD in \eqref{eq:asym_norm} takes the so-called sandwich form: $\Sigma = A^{-1}SA^{-1}$, where
\begin{equation}\label{sandwich}
A = \nabla^{2}F(x^{*}), \ S = \E\left([\nabla f(x^{*}, \xi)][\nabla f(x^{*}, \xi)]\T \right).
\end{equation} 
Three primary approaches for estimating the covariance structure have been developed in the stochastic approximation literature.

The first method approximates the distribution via bootstrap resampling, as explored in \citep{fang2018online, li2018statistical,zhong2023online}. To obtain a consistent covariance estimate or achieve stable confidence intervals, a large number of bootstrap sequences are required, each of which modifies the original SGD sequence by adding perturbations and recomputing the gradients at every iteration. Consequently, the total cost becomes a substantial multiple of the original SGD run, which makes this approach computationally expensive and less suitable for SGD-related tasks. In this work, we aim to develop a method that operates entirely on a single SGD trajectory.

The second approach is the plug-in estimator introduced by \cite{chen2020statistical}, which separately estimates $A$ and $S$ in the sandwich form of the asymptotic covariance. The key idea is to approximate these matrices using empirical averages. Specifically, the estimators are given by
$\widehat{A}_n=\frac{1}{n}\sum_{i=1}^n \nabla^{2}f(x_{i-1},\xi_i),$
    $\widehat{S}_n =\frac{1}{n}\sum_{i=1}^n [\nabla f(x_{i-1}, \xi_i)][\nabla f(x_{i-1}, \xi_i)]\T $. The resulting covariance estimator $\widehat{A}_n^{-1}\widehat{S}_n\widehat{A}_n^{-1}$ is consistent, with a convergence rate of $n^{-\alpha/2}$, and can be computed in an online manner. Similar ideas have been applied in various settings, including constrained stochastic optimization, online decision-making, and non-differentiable problems \cite{Na2025Statistical, chen2021First, chen2021statistical}. However, the practical implementation of this estimator presents challenges: it requires access to the stochastic Hessian, which is often unavailable, and involves matrix computations with a computational complexity of $\mathcal{O}(d^3)$, making it inefficient for high-dimensional scenarios.

The third approach is the batch-means estimator, a Hessian-free method that relies solely on SGD iterates. Unlike the plug-in estimator, which is based on the sandwich formula, the batch-means method directly estimates the variance by analyzing the variability in the SGD sequence itself. The origins of batch-means methods with fixed batch size can be traced back to long-run variance estimation for time-homogeneous Markov chains \citep{glynn1990simulation, glynn1991estimating, damerdji1991strong, geyer1992practical}. \cite{chen2020statistical} introduced a batch-means method with increasing batch sizes to account for the complex correlation structure inherent in SGD. However, the batch construction in their approach cannot be updated recursively, as it depends on the total number of iterations. This means one would need to store all past iterates and, when new data arrive, reconstruct the batches using all previous iterations and then compute the estimator. To address this limitation, \cite{zhu2023online} developed a fully online batch-means covariance estimator that can be updated on the fly.  While the batch-means approach is  computationally efficient and practically feasible, it suffers from a relatively slow convergence rate of
$n^{(\alpha-1)/4}$, for both online and offline methods.

The same batching idea has been extended to broader contexts. For instance, \cite{Jiang2025Online} demonstrated that the online batch-means covariance estimator remains valid in nonsmooth and potentially non-monotone (including certain nonconvex) settings, preserving the same convergence rate of $n^{(\alpha-1)/4}$. Meanwhile, \cite{kuang2025online} proposed a weighted sample covariance estimator for sketched Newton iterates and provided theoretical guarantees within the framework of second-order optimization.

In this paper, we modify the Hessian-free batch-means estimator and provide a de-biased version that retains its computational advantages while offering improved convergence.


\section{METHOD}\label{sec:method}
\subsection{Motivation for Bias-Reduction}\label{sec:motivation}
We now take a closer look at the batch-means estimator and provide some intuition for our proposed de-biasing idea. For simplicity, we consider a general one-dimensional sequence $\{X_i\}_{i \geq 1}$.

As discussed in the previous section, rather than relying on the sandwich form of $\Sigma$, the batch-means method aims to estimate the variability of the average $\bar{X}_n = n^{-1} \sum_{i=1}^n X_i$ directly from the sequence $\{X_i\}_{i \geq 1}$. Specifically, it targets the quantity
$$\sigma_{n} = n\textnormal{Var}\bar{X}_{n},$$ which captures the dispersion of the running average and converges to the asymptotic covariance matrix $\Sigma$ in the SGD setting. 
In the independent and identically distributed (\emph{i.i.d.}) setting, a natural choice for estimating the variance of the sample mean is the sample covariance:  
$$\hat\sigma_{n} = \frac{\sum_{i=1}^{n}\left(X_{i} - \bar{X}_{n}\right)^2}{n}.$$ However,  in the case of  SGD, the iterates ${X_i}$ are highly correlated, which leads to substantial underestimation when using this simple form. To account for this correlation, the online batch-means estimator \cite{zhu2023online} modifies the sample covariance by incorporating local batching:
\[\hat\sigma_{n, bm} =  \frac{\sum_{i=1}^{n}\left(\sum_{j = i-l_{i}+1}^{i}X_{j} - l_{i}\bar{X}_{n}\right)^2}{\sum_{i=1}^{n}l_{i}},\]
where $1 \leq l_i \leq i$ denotes the size of the $i$-th batch in their paper. The choice of batch sizes is guided by the correlation structure of the SGD iterates. While this estimator is computationally efficient and Hessian-free, its convergence can be relatively slow due to bias introduced in its formulation. That is, the batch-means estimator $\hat\sigma_{n, bm}$ still builds on the structure of the biased sample variance, albeit with adjusted weights. This motivates our pursuit of a more principled, de-biased approach.

Instead of adjusting the biased sample covariance, we propose constructing an estimator that more directly reflects the theoretical definition of $\sigma_{n}$. 
In particular, observe the identity:
	   \begin{equation}\label{varexpansion}
	       \begin{split} 
           \sigma_{n} =   \frac{1}{n}\sum_{i=1}^{n}\mathbb{E}\Bigg[2(X_{i}-\E(X_{i}) )\sum_{k=1}^{i}(X_{k}-\E(X_k))\\ - (X_{i}-\E(X_{i}))^2\Bigg],
           \end{split}
	   \end{equation}
    which offers an alternative way of characterizing the variance of the average. This motivates the following empirical estimator $\hat\sigma_{n, db}$: 
\begin{equation}
	       \begin{split} 	
           \frac1n  \sum_{i=1}^n \Bigg[    2(X_i - \bar X_n)\bigg(\sum_{k=i-\ell_{i}+1}^{i} X_k - \ell_i\bar X_n \bigg)  \\ 
           - (X_i - \bar X_n)^2    \Bigg],
            \end{split}
	   \end{equation}
where $\ell_{i}$ again denotes the size of the $i$-th batch in this paper. Building on this intuition, the following subsection formally defines the de-biased estimator for the limiting covariance matrix $\Sigma$ of averaged SGD. Our proposed estimator is constructed to mirror the expansion \eqref{varexpansion} as closely as possible, combined with a principled choice of the batch size $\ell_{i}$ to guarantee the theoretical convergence rate.

\subsection{De-biased Estimator }
Recall that  the sequence of SGD iterates $\{x_{i}\}_{i\ge 1}$ is generated by the recursion in \eqref{eq:SGDite}  at the $i$-th iteration, we define the batch 
$$B_{i} = \{i-\ell_{i}+1, \dots, i\}, \quad |B_{i}| = \ell_{i},$$ and propose the de-biased estimator: $\widehat{\Sigma}_{n}$ as:  
        \begin{equation}\label{eqn:est}
		\begin{split} 
			\widehat{\Sigma}_{n}  =  	\frac1n  \sum_{i=1}^n \Bigg[  (x_i - \bar x_n) \bigg(\sum_{k=i-\ell_{i}+1}^{i} x_k - \ell_i\bar x_n \bigg)\T \\
			 +   \bigg(\sum_{k=i-\ell_{i}+1}^{i} x_k - \ell_i\bar x_n \bigg) (x_i - \bar x_n)\T \\
			 - (x_i - \bar x_n)(x_i - \bar x_n)\T  \Bigg],
		\end{split}
	\end{equation}
    where $\bar{x}_{n} = n^{-1}\sum_{i=1}^{n}x_i$.
Our theoretical analysis in Section~\ref{sec:theory} establishes that this estimator is consistent for $\Sigma$ when the batch sizes satisfy $\ell_i = \mathcal{O}(i^\alpha \log i)$ with $\alpha \in (0.5, 1)$.  In practice, general choices of $\ell_i$ often prevent online updates due to overlapping batches and variable batch sizes. To address this, we introduce a block-based batching scheme that ensures each batch has size $|B_i| = \mathcal{O}(i^{\alpha} \log i)$ while remaining fully compatible with online implementation.

\begin{remark}
For stationary time series, a conceptually related bias-reduction method was proposed by \cite{xiao2011single} to estimate the long-run covariance. Inspired by their work, we develop the de-biased estimator in the substantially more intricate non-stationary and time-inhomogeneous regime. The motivation in \cite{xiao2011single} relied on the relationship between auto-covariance functions and the long-run covariance, which is intrinsically limited to stationary processes. In contrast, the intuition of our estimator comes from the non-asymptotic covariance of the sample average that applies to much broader settings.

Our bias-reduction approach and selection of batch size also differ from those in \cite{xiao2011single}, where the authors used an increasing threshold sequence to determine which batches are retained in the estimator (see their Equation (12)). Our method is free of this additional threshold parameter. Instead, we analytically design the batch size $\ell_i$ to directly balance the bias-variance tradeoff.
\end{remark}

\subsection{Practical Implementation}

{\textbf{Block-based Batching Scheme.}}
	\begin{figure} 
			\centering \includegraphics[width=0.45\textwidth]{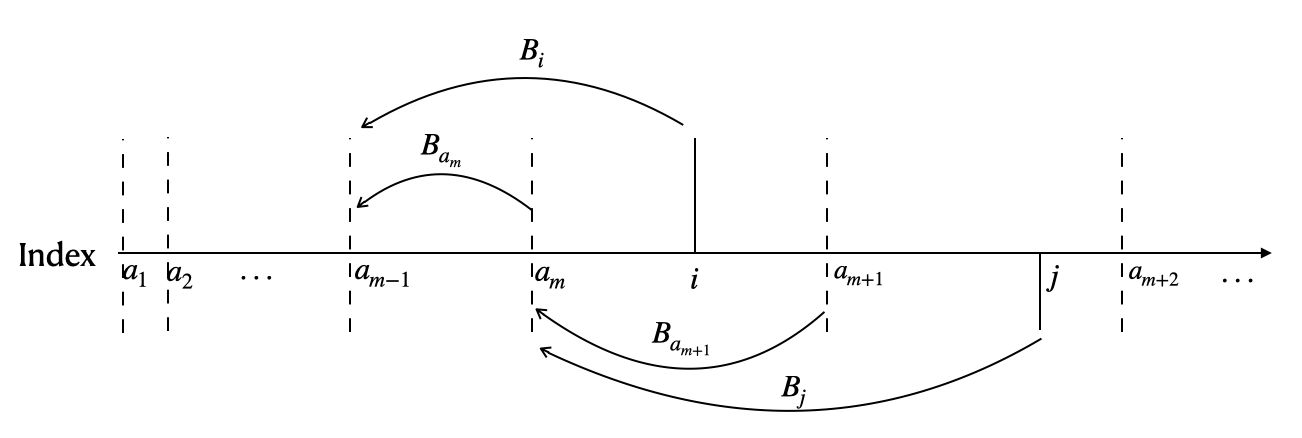}	
            \caption{Illustration for Block-based Batching Scheme}
			\label{fig:example} 
		\end{figure}     
We define a sequence of indices $\{a_{m}\}_{m\ge 1}$ with $a_{1} = 1$, and subsequent values satisfy the condition:
\begin{equation}\label{eq:condition_am}
			a_{m} - a_{m-1} + 1= \lfloor a_{m}^{\alpha}\log(a_{m})\rfloor, m\ge 2.
		\end{equation} 
That is to say, we start a new block when the current block length reaches $\lfloor i^{\alpha}\log(i)\rfloor$. At each iteration $i$, find the unique index $m$ such that $a_{m}\le i< a_{m+1}$, and define the batch 
\[B_{i} = \{a_{m-1}, a_{m-1}+1, ..., a_{m}, ..., i\},\] 
which corresponds to $\ell_i = |B_i| = i - a_{m-1}+1$. This batching strategy always incorporates the two most recent blocks and satisfies the desired scaling for $\ell_i$; see figure \ref{fig:example} for an illustration. 

\begin{proposition}\label{onlinebatch}
    Let the batch $B_i$  be constructed using the Block-based Batching Scheme with indices sequence $\{a_{m}\}_{m\ge 1}$ as described above in \eqref{eq:condition_am}. Then the batch size satisfies 
    \[\lfloor i^{\alpha}\log(i)\rfloor\le |B_{i}|\le 2\lfloor i^{\alpha}\log(i)\rfloor\]
    and hence $|B_i| = \mathcal{O}(i^{\alpha}\log i)$. 
\end{proposition}

  
\textbf{Online Update.} In practice, we can fix $a_{1} = 1$, and then recursively determine each subsequent  $a_{m}$ by incrementing its value until the condition in  \eqref{eq:condition_am} is satisfied. This enables the batch structure to be constructed in a fully online fashion, and thus supports the recursive update of the de-biased estimator.
Specifically, we update the batch sum and batch length as follows. Recall that at the $i$-th step, the batch $B_i$ contains of two parts: the previous block $\{a_{m-1}, \dots, a_{m}-1\}$, and the current block $\{a_{m}, \dots, i\}$. 
Let $S_0$ denote the sum of previous block, and $S_1$ the sum of current block.  Record  the starting index of the current block $a_{m}$ as a, 
\begin{itemize}
    \item If $i-a+1< \lfloor i^{\alpha}\log(i)\rfloor$, we simply add the new iterate $x_i$ to the current batch and all we need to update is  $S_1 = S_1 + x_{i}$ 
    \item If $i-a+1\ge \lfloor i^{\alpha}\log(i)\rfloor$, we reset the batch by discarding the previous block. In this case, we update: $S_0 = S_1$ and $S_1 = x_{i}$, $a = i$.
\end{itemize}  
Then the batch sum $\sum_{k=i-\ell_{i}+1}^{i}x_k = S_0+S_1$, and the batch size $\ell_{i} = i - a + 1$,  can both be updated recursively. 
If we expand the de-biased estimator in \eqref{eqn:est},  we obtain:
\begin{equation*}
    \begin{split}
        n\widehat{\Sigma}_{n}
         =& \sum_{i=1}^n x_i \left(\sum_{k=i-\ell_{i}+1}^{i} x_k\right)\T  + \sum_{i=1}^n \left(\sum_{k=i-\ell_{i}+1}^{i} x_k\right) x_i \T \\
        &- \sum_{i=1}^n\left(\ell_i x_i + \sum_{k=i-\ell_{i}+1}^{i} x_k \right) \bar x_n\T  \\
       & - \bar x_n\sum_{i=1}^n\left(\ell_i x_i + \sum_{k=i-\ell_{i}+1}^{i} x_k \right) \T  \\
         &+ \sum_{i=1}^n(2\ell_{i}+1)\bar x_n\bar x_n\T - \sum_{i=1}^{n}x_i x_i\T .
    \end{split}
\end{equation*}
We observe that the estimator can be computed recursively, provided we maintain recursive updates of the averaged SGD, batch sum, and batch size. We summarize the details in Algorithm \ref{alg:general}, which demonstrates that the covariance matrix estimator can indeed be computed recursively.   This implies that, as the SGD algorithm progresses from the $i$-th to $(i+1)$-th iterate, the covariance matrix can be updated simultaneously with minimal computation. The update requires only the new SGD iterate and a few quantities—such as the batch sum and batch size—carried over from the previous step. The computational cost per iteration is $\mathcal{O}(d^2)$, which is minimal, given that we are estimating a $d\times d$ matrix.

\begin{algorithm}[tb]
\caption{Recursive update for the de-biased covariance estimator}\label{alg:general}
\textbf{Input}: Step sizes $\{\eta_{i}\}_{i\ge 1}$, initialization $x_{0}, a = 1, S_{0} = S_{1} = \bar{x} = P = W = Q  = q = 0$\\ 
\begin{algorithmic}[1] 
\FOR{ $i = 1, 2, \dots$}
\STATE Sample $\xi_{i}\sim \Pi$
\STATE $x_{i} = x_{i-1} - \eta_{i}\nabla{f(x_{i-1}, \xi_{i})}$
 \STATE  $\bar{x} = ((i-1)\bar{x} + x_{i})/i$
 \IF {$i - a + 1 >= \lfloor i^{\alpha}\log(i)\rfloor$}
\STATE   $a = i$	
\STATE $S_{0} = S_{1}, S_{1} = x_{i}$
\ELSE
\STATE  $S_{1} = S_{1} + x_{i}$
\ENDIF
\STATE  $l = i - a + 1$
\STATE  $S =  S_{0} + S_{1}$\
\STATE  $P = P +  x_{i}S\T ; \  W = W + l x_{i} + S; \ Q = Q+x_{i} x_{i} \T ; \ q = q+2l+1$
\STATE  $V =P + P\T  - W \bar{x}\T  - \bar{x}W\T +  q\bar{x}\bar{x}\T   - Q$
\STATE  \textbf{Output (if necessary)}:
			$\widehat\Sigma_{i} =V/i$ 
            \ENDFOR
\end{algorithmic}
 \end{algorithm}

		\section{THEORETICAL GUARANTEE}\label{sec:theory}


       The batch size $\ell_i$ plays a critical role in constructing an accurate estimator. As mentioned in Section~\ref{sec:method}, an effective choice is to let $\ell_i$ grow at the rate $\mathcal{O}(i^{\alpha} \log i)$, and we propose a practical, fully online batching scheme that satisfies this condition. In this section, we first build intuition for this choice by analyzing a simple mean estimation model as a motivating example in Section \ref{sec:mean est}. Then, in Section~\ref{sec:general_theory}, we establish that this batch size leads to a consistent estimator in the general setting. Specifically, we provide an upper bound on the mean squared error (MSE) $\mathbb{E} \|\widehat{\Sigma}_{n} - \Sigma\|^2$, and show that the bound is tight.

  
		\subsection{Intuition Behind the Batch Size Choice}\label{sec:mean est}

        Let $\{\xi_{i}\}_{i\in\N}$ be a sequence of \emph{i.i.d.} random variables from the model
		$\xi = x^{*} + e,$
		where $x^{*}\in\R$ is the true population mean of interest, and $e$ is a Gaussian random error with $\E [e] = 0$ and $\E[e^2] = \sigma < \infty$.  Consider the squared loss function $f(x, \xi) = (\xi - x)^2/2$, and generate the $i$-th SGD iterate by
\begin{equation}\label{eq:mean_est_SGD}
x_{i} = x_{i-1} - \eta_{i}(x_{i-1} - \xi_{i}),
\end{equation}
where $\eta_i = \eta i^{-\alpha}$ with $\eta > 0$ and $\alpha \in (0.5, 1)$.  
For simplicity of illustration, we assume $x^*=x_0=0$ in this subsection, and consider an oracle estimator 
\begin{equation}\label{oracle}
\hat{\sigma}_n = \frac{1}{n}\sum_{i=1}^n[2x_i(\sum_{k=i-\ell_i+1}^{i}x_k) -x_i^2   ].
\end{equation}
Unlike the construction in \eqref{eqn:est}, we did not subtract the sample mean $\bar{x}_n$ in the estimator $\hat{\sigma}_n$ because $\E x_i =0$ for all $i$.
Due to the Gaussianity and linearity of the model, we can derive closed-form expressions for the oracle de-biased estimator $\hat{\sigma}_n$. This tractability enables an explicit analysis of the batch size $\ell_i$ and the precise order of the resulting estimation error.


We begin by analyzing how the error between $\hat{\sigma}_n$ and the finite-sample variance $\sigma_{n} = \textnormal{Var}(\sqrt{n}(\bar{x}_{n})$   depends on the choice of $\ell_i$, as summarized in the following proposition. In practice, since we work with finite samples, a bound on the error between $\hat{\sigma}_n$ and $\sigma_n$ not only implies convergence to the limiting variance $\Sigma$, but also provides a more practical assessment of the estimator’s performance in finite-sample settings. 

\begin{proposition}\label{bias}
Consider the SGD iterates $\{x_{i}\}_{i=1}^{n}$ defined by \eqref{eq:mean_est_SGD}, with step sizes $\eta_i = \eta i^{-\alpha}$ for some $\eta > 0$ and $\alpha \in (0.5, 1)$. The proposed oracle de-biased estimator $\hat{\sigma}_n$ defined in \eqref{oracle} converges to the 
true finite-sample variance  $\sigma_{n} = \textnormal{Var}(\sqrt{n}\bar{x}_{n})$  
with the following bound:
$$ |\E (\hat{\sigma}_n-\sigma_n) |\lesssim \frac{1}{n}\sum_{i=1}^n \exp\{ -\eta i^{-\alpha}\ell_i\}.$$ 
Moreover, if the batch size satisfies $\ell_i \lesssim i^\alpha$, this convergence rate is tight, i.e.,
$$ |\E (\hat{\sigma}_n-\sigma_n)| \asymp \frac{1}{n}\sum_{i=1}^n \exp\{ -\eta i^{-\alpha}\ell_i\}.$$ 
\end{proposition}
\begin{remark}[Choice of batch size $\ell_{i}$]
Proposition \ref{bias} provides two important insights regarding the batch size. 
\begin{itemize}
    \item If the batch size grows too slowly, specifically if  $\ell_i \lesssim i^{\alpha}$, the estimator incurs a non-negligible bias:
    $$ | \E (\hat{\sigma}_n-\sigma_n)| \gtrsim \frac{1}{n}\sum_{i=1}^n \exp( -\eta ) = \exp(-\eta),$$
    which does not vanish as $n \to \infty$.
    \item As long as the batch size $\ell_i$ grow slightly faster than $i^{\alpha}$, for example, 
\begin{equation}\label{batchc}
    \ell_i \ge   C_{\eta}i^{\alpha}\log i
\end{equation} 
for some universal constant $ C_{\eta}$,  the estimator becomes asymptotically unbiased since
$$ |\E (\hat{\sigma}_n-\sigma_n)| \lesssim \frac{1}{n}\sum_{i=1}^n \exp( -\eta  C_{\eta}\log i ) \rightarrow 0.$$
\end{itemize}  
This justifies the recommended batch size scaling of order $\mathcal{O}(i^{\alpha} \log i)$.
\end{remark}


 \begin{figure}
\centering
{\includegraphics[width=0.45\textwidth]{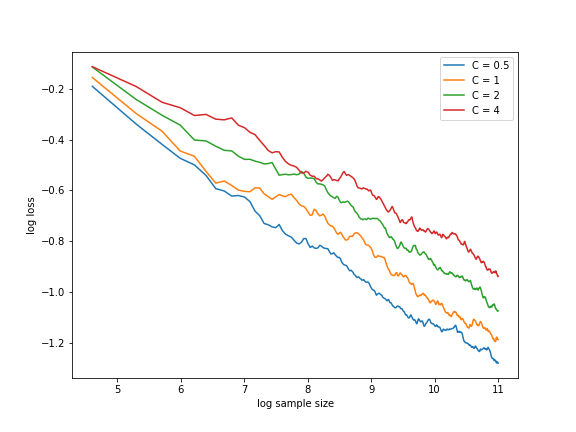}}
\caption{Log-log Plots of Estimation Errors for Different Values of $C$ in the Batch Sequence $\{a_{k}\}$, under the Linear Model with $d=1$.}
\label{compareC} 
\end{figure}

\begin{figure}[h]
\centering
 \includegraphics[width=0.23\textwidth]{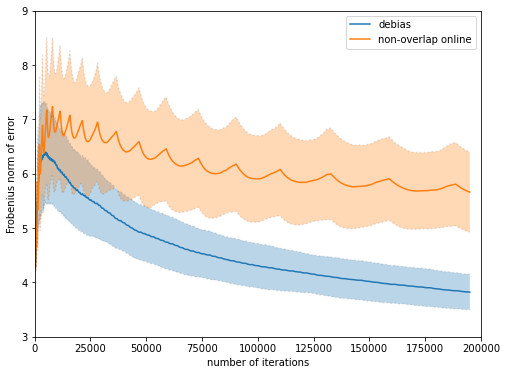}   
  \includegraphics[width=0.23\textwidth]{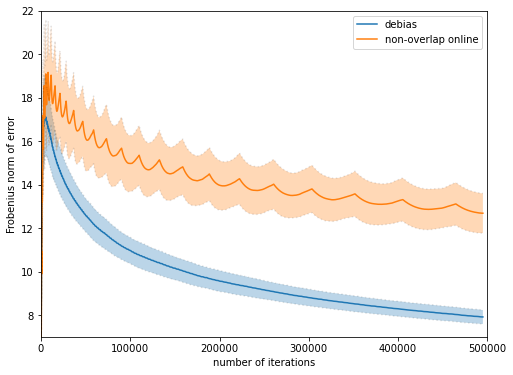} 
  \includegraphics[width=0.23\textwidth]{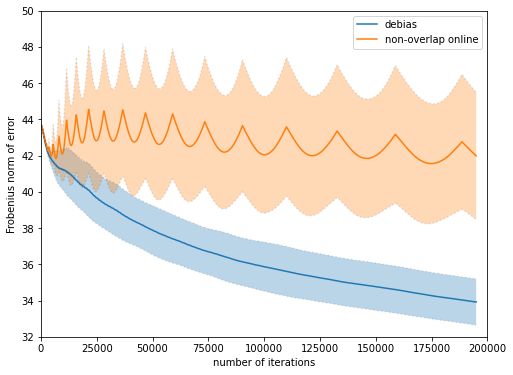} 
  \includegraphics[width=0.23\textwidth]{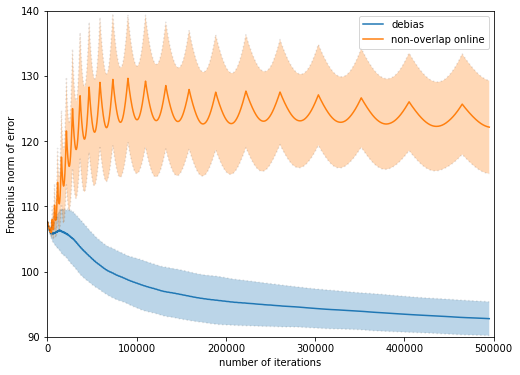}
  \includegraphics[width=0.23\textwidth]{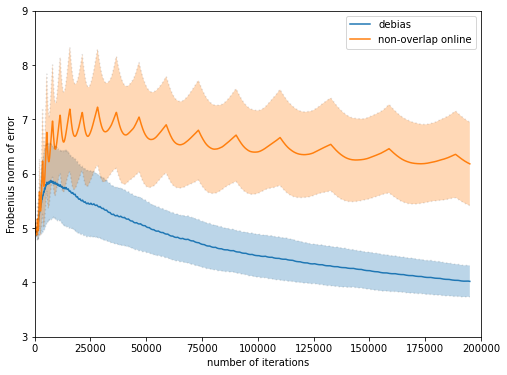}   
  \includegraphics[width=0.23\textwidth]{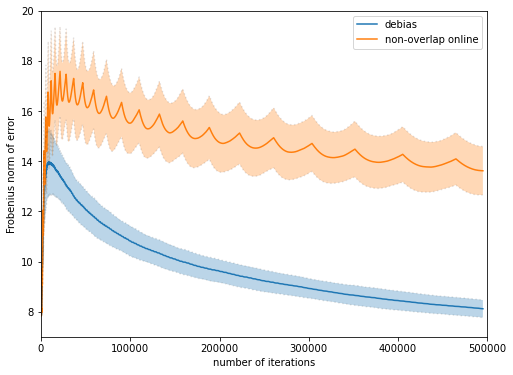} 
\caption{The Estimation Error at Each Iteration. Left: $d=20$. Right: $d=50$. Top: Linear Regression Model. Middle: Logistic Regression Model. Bottom: Expectile Regression Model. The Error Band Represents One Standard Deviation. }
\label{fig:MSEd20} 
\end{figure}

To align with the broader theoretical literature, we establish the consistency by bounding the error between $\hat{\sigma}_n$ and the limiting variance $\Sigma$ in \eqref{eq:asym_norm}. By elementary calculations, one can show that the limiting variance equals $\sigma$, the variance of the noise, in the context of \eqref{eq:mean_est_SGD}. 

 \begin{theorem}\label{thm:mean} Consider the SGD iterates $\{x_{i}\}_{i=1}^{n}$ defined by \eqref{eq:mean_est_SGD}, with step sizes $\eta_i = \eta i^{-\alpha}$ for some $\eta > 0$ and $\alpha \in (0.5, 1)$.   Let $\hat{\sigma}_n$  be the proposed oracle de-biased estimator in \eqref{oracle}  with $\ell_i = C_\eta i^{\alpha}\log i$ for some constant $C_{\eta}>\eta^{-1}$. Then the convergence rate of MSE satisfies
	   	$$\E(\hat{\sigma}_n - \sigma)^2 \asymp n^{-1+\alpha}\log{n}.$$
\end{theorem}
In Theorem \ref{thm:mean}, the order of MSE is $ n^{-1+\alpha}\log{n}$ as long as $\eta C_\eta >1$. Since $\eta$ and $C_{\eta}$ are free to choose in practice, this condition is easily achievable. We will assume this holds in the rest of our paper. For values of \( \alpha \) close to \( 1/2 \), the MSE rate approaches \( n^{-1/2} \log n \), which is significantly sharper than the rates of existing online covariance estimators. For instance, the non-overlapping online estimator in \cite{zhu2023online} and the batch-means estimator in \cite{chen2021statistical} achieve rates of \( n^{-1/3} \) and \( n^{-1/4} \), respectively.


\begin{table*}[h] 
\centering
\caption{Comparison of De-biased and Online BM Estimators for Different Models with $d=5$. Standard Deviations are Reported in Parentheses.}
\label{tb1}
\begin{tabular}{l|c|c|c|r} 
& & $n=15000$  &  $n=30000$  & $n=60000$ \\
\hline
\textbf{Linear} & De-biased & 1.55 (0.36) & 1.39 (0.35) & 1.21 (0.26) \\ 
               & Online BM& 1.79 (0.45) & 1.76 (0.48) & 1.65 (0.42) \\
\hline
\textbf{Logistic} & De-biased & 10.62 (0.92) & 9.78 (0.89) & 8.99 (0.92) \\ 
                  & Online BM & 11.08 (1.07) & 10.72 (1.30) & 10.21 (1.64) \\
\hline
\textbf{Expectile} & De-biased & 1.87 (0.34) & 1.68 (0.28) & 1.51 (0.24) \\ 
                   & Online BM & 2.19 (0.40) & 2.20 (0.49) & 2.15 (0.49) \\
\end{tabular}
\end{table*}

\begin{table*}[h]
\centering
\caption{Comparison of De-biased and Online BM Estimators for Different Models with $d=20$. Standard Deviations are Reported in Parentheses.}
\label{tb2}
\begin{tabular}{l|l|c|c|c|}
  &   & $n=50000$ & $n=100000$ & $n=200000$ \\
\hline
\textbf{Linear}
    & De-biased             & 4.96 (0.52) & 4.34 (0.41) & 3.82 (0.33) \\
    & Online BM    & 6.51 (0.95) & 5.95 (0.82) & 5.66 (0.74) \\
\hline
\textbf{Logistic}
    & De-biased             & 38.19 (1.33) & 36.01 (1.33) & 33.92 (1.28) \\
    & Online BM    & 43.50 (3.16) & 42.57 (3.34) & 41.98 (3.51) \\
\hline
\textbf{Expectile}
    & De-biased             & 5.10 (0.49) & 4.53 (0.39) & 4.02 (0.29) \\
    & Online BM    & 6.89 (0.93) & 6.45 (0.80) & 6.18 (0.77) \\
\end{tabular}
\end{table*}

\begin{table*}[h]
\centering
\caption{Comparison of De-biased and Online BM Estimators for Different Models with $d=50$. Standard Deviations are Reported in Parentheses.}
\label{tb3}
\begin{tabular}{l|l|c|c|c|}
  &   & $n=125000$ & $n=250000$ & $n=500000$ \\
\hline
\textbf{Linear}
    & De-biased             & 9.96 (0.52) & 8.63 (0.37) & 7.91 (0.32) \\
    & Online BM    & 14.37 (1.22) & 13.31 (0.98) & 12.68 (0.91) \\
\hline
\textbf{Logistic}
    & De-biased             & 96.07 (3.69) &94.09 (2.90) & 92.78(2.59) \\
    & Online BM    & 124.77 (8.48) & 122.90 (6.96) & 122.14 (7.12) \\
\hline
\textbf{Expectile}
    & De-biased             & 9.94 (0.49) & 8.76 (0.38) & 8.12 (0.35) \\
    & Online BM    & 15.10 (1.22) & 14.15 (1.02) & 13.62 (0.97) \\
\end{tabular}
\end{table*}

     	\subsection{General Convergence Analysis of the De-biased Estimator}\label{sec:general_theory}
        The convergence in the general setting is significantly more difficult to analyze due to the nonlinearity and complex dependency structure of the iterates. Nevertheless, in this section, we will show that the upper bound on the MSE: $\mathbb{E} \|\widehat{\Sigma}_{n} - \Sigma\|^2$ matches the exact rate established in Theorem~\ref{thm:mean}.
        
     Before the main theorem, we introduce some basic assumptions concerning convexity, Lipschitz continuity, boundness, and other regularity conditions.

     \begin{assumption} \label{as1}
	The objective function $F(x)$ is continuously differentiable and strongly convex with parameter $\mu>0$. That is, for any $x_1$ and $x_2$,
	$$F(x_2)\ge F(x_1)+\Braket{ \nabla F(x_1),x_2-x_1}+\frac{\mu}{2}\|x_1-x_2\|^2_2.$$
	Further assume that $\nabla^2 F(x^*)$ exists. 
\end{assumption} 
\begin{assumption}\label{as2}
	The function $f(x,\xi)$ is continuous differentiable with respect to $x$ for any $\xi$ and $\|\nabla f(x,\xi)\|_2$ is uniformly integrable for any $x$. 
\end{assumption}

\begin{assumption}\label{as3}
The gradient noise satisfies $\mathbb{E}_{\xi} \|\nabla f(x^*,\xi) \|_{2}^{q} < \infty$ for some $q\ge 8$, and $\nabla f(x,\xi)$ is stochastic Lipschitz continuous with parameter $L$, i.e., for any $x_1$ and $x_2$,
	$$
	(\mathbb{E}_{\xi} \|\nabla f(x_1,\xi)-\nabla f(x_2,\xi)\|_2^q)^{\frac{1}{q}}\leq L  \|x_1-x_2\|_2.$$

\end{assumption} 
Assumptions \ref{as1}-\ref{as3} are relatively mild and commonly appear in the literature on convex optimization and statistical inference using SGD, as well as in prior work on estimating the limiting covariance matrix \cite{chen2021statistical, lee2022fast, zhu2023online,zhu2024high}.

		
		   \begin{theorem}\label{thm:general}
		Under Assumption \ref{as1}-\ref{as3}, consider the SGD iterates $\{x_{i}\}_{i=1}^{n}$  with step sizes $\eta_i = \eta i^{-\alpha}$ for some $\eta > 0$ and $\alpha \in (0.5, 1)$.  Let $\widehat{\Sigma}_n$ be the de-biased covariance estimator defined in \eqref{eqn:est}, using batch sizes $\ell_i = C_\eta i^{\alpha} \log i$ for a constant $C_\eta > \eta^{-1}$. Then $\widehat{\Sigma}_n$ converges to the limiting covariance matrix $\Sigma$ at the following rate:
		$$\E \| \widehat\Sigma_n - \Sigma\|^2 \lesssim n^{-1+\alpha}\log{n}.$$
		
			\end{theorem}
   \begin{remark}
   The convergence rate established in Theorem \ref{thm:general} is notable for two key reasons. 
   
   First, when combined with Theorem \ref{thm:mean}, which analyzes the simple mean estimation setting, it demonstrates that the rate $n^{-1+\alpha}\log{n}$ is tight --- that is, it cannot be improved in general. Theorem \ref{thm:mean} shows that this rate is achieved exactly in the simplest setting, implying that our upper bound in the general model is sharp for the proposed de-biased estimator.

   Second, to the best of our knowledge, this rate is the best known in the literature for online covariance estimation under SGD without requiring information from the Hessian.  Specifically, the batch-means estimators studied in \cite{chen2021statistical} and \cite{zhu2023online} yield upper bounds of order $\mathcal{O}(n^{(-1+\alpha)/4})$ for  $\E\|\widehat\Sigma_n - \Sigma\|$, which is is significantly slower than our bound: 
    \[\E \| \widehat\Sigma_n - \Sigma\|\leq  \sqrt{\E \| \widehat\Sigma_n - \Sigma\|^2} \lesssim  n^{(-1+\alpha)/2}\sqrt{\log{n}}.\] Moreover, while the plug-in method in \cite{chen2021statistical} achieves a faster rate of $\mathcal{O}(n^{-\alpha/2})$, it relies heavily on repeated Hessian computation and matrix inversion, making it computationally expensive. In contrast, our estimator achieves a comparable rate—particularly when $\alpha \approx 1/2$, while remaining fully online and computationally efficient, without requiring any access to second-order information. 
   \end{remark}

	\section{NUMERICAL STUDY}\label{sec:simulation}

  In this section, we assess the empirical performance of the de-biased estimator across different settings. In the subsequent numerical experiments, $\xi_i=(a_i,b_i), i=1,2,...$ are \emph{i.i.d.} random vectors, where $a_i \sim \mathcal{N}(0,{\bf I}_d)$, and $b_i$ are drawn from different distributions based on the linear, logistic and expectile regression models. 
  

  For the linear regression model, $b_i \sim \mathcal{N}(a_i\T x^*,1)$, and the loss function is defined as the squared loss 
$$f(x,\xi_i  = (a_{i}, b_{i}))=(a_i\T  x-b_i)^2/2. $$
For the logistic regression model, $b_i \in \{1,-1\}$ is generated from a Bernoulli distribution with the probability given by $\mathbb{P}(b_i|a_i)=1/(1+\exp (-b_ia\T _ix^*))$. We use the logit loss 
$$f(x,\xi_i = (a_{i}, b_{i}))=\log (1+\exp (-b_ia\T _ix)).$$
Notice that both loss functions are the corresponding negative log-likelihood.

For expectile regression \citep{newey1987asymmetric}, which has been extensively studied and applied in the statistical and economic literature \citep{efron1991regression, taylor2008estimating,daouia2024expectile}, we specify the model as follows. Rewriting $x$ as $(x,x_0) \in \R^{d+1}$ where $x_0$ is a univariate intercept term. The loss function for $\tau$-th expectile regression is defined as 
$$f(x,x_0,\xi) = |\tau - 1_{\{b < a \T x+x_0\}}|(b- a\T x-x_0)^2, \ 0<\tau<1,$$
and $b_i \sim \mathcal{N}(a_i\T x^*,1)$.

The true parameter $x^*$ is an arithmetic sequence from $0$ to $1$ with length $d$.  We evaluate and report the estimation error $\|\wh{\Sigma}_{n}-\Sigma\|_F$.
All reported results are averaged over 500 independent runs. The learning rate in SGD is chosen as $\eta_i=0.5i^{-0.505}$,  following the choice in \cite{zhu2023online}.

For the linear regression model, 
a straightforward derivation yields $A=S={\bf I}_d$ and thus $\Sigma = {\bf I}_d$. For the logistic regression model, 
since an explicit form for $A$ and $S$ is challenging to derive, we empirically estimate the covariance matrix $\Sigma$ using Monte Carlo simulations.

In the expectile regression setting, 
the unique optimizer of its objective function is $(x^*,x_0^{\tau})$, where $x_0^{\tau}$ is the $\tau$-th expectile of a univariate standard Gaussian distribution. Elementary calculus yields 
$$A=2[(1 - \tau) + (2\tau - 1)\, \Phi(-x_0^{\tau})]{\bf I}_d,$$ $$S = 4(\tau^2 \alpha_+ + (1 - \tau)^2 \alpha_-){\bf I}_d,$$ 
with 
$$\alpha_+ = \Phi(-x_0^{\tau})\,(1 + (x_0^{\tau})^2) - x_0^{\tau}\phi(x_0^{\tau})$$
$$\alpha_- = (1 + \Phi(x_0^{\tau}))\,(1 + (x_0^{\tau})^2) +x_0^{\tau}\phi(x_0^{\tau}),$$ 
and $\Phi$ and $\phi$ denote the cumulative distribution function and probability density function of the standard normal distribution, respectively. Therefore, we have the analytic expression for $\Sigma = A^{-1}S A^{-1}$.



 We begin by investigating the impact of batch size on the convergence rate of the de-biased estimator in a simple one-dimensional linear model, where the loss is defined as the absolute error of the estimated variance.  For the online approach of the de-biased estimator, the batch sequence $\{a_k\}$ is chosen as $a_k = C a^0_k$ for some constant $C$, where
\begin{equation}
	a^0_{m} - a^0_{m-1} + 1= \lfloor (a^0_{m})^{\alpha}\log(a^0_{m})\rfloor, m\ge 2.
\end{equation}
 In the appendix, we will prove that this choice of the batch sequence also ensures $|B_i| = \mathcal{O}(i^{\alpha}\log i)$, similar to Proposition \ref{onlinebatch}.
 Figure \ref{compareC} shows that smaller values of $C$ lead to a slight reduction in estimation error. Overall, the convergence rate—reflected in the slope—remains largely unaffected by variations in $C$ within a reasonable range, demonstrating the robust convergence behavior of the de-biased estimator. Unless otherwise specified, we choose $C=0.5$ in the following simulations.

We compare the estimation error of our de-biased method with that of the online batch-means method (non-overlap version) proposed by \citet{zhu2023online}, which is, to the best of our knowledge, the only Hessian-free online estimator available in the literature. Figure \ref{fig:MSEd20} displays the Frobenius norm of the estimation error plotted against the number of iterations for linear, logistic, and expectile regression models. The results highlight a noticeably sharper convergence rate and substantially lower estimation error of the de-biased estimator. Additional plots for other settings are provided in the supplementary material. As shown in Tables \ref{tb1}-\ref{tb3}, the de-biased approach consistently achieves higher accuracy and stability than the non-overlap online method across different model types and dimensional settings as the number of iterations increases.

 In the appendix, we also present numerical experiment results to demonstrate the application of confidence interval construction. We evaluate the performances under the same settings. Post the estimation of the limiting covariance matrix with $\widehat{\Sigma}_n$, we construct the $(1-q)100\%$ confidence interval for the $j$-th coordinates of $x^*$ as
\begin{align*}
  \left[(\bar{x}_{n})_{j} -z_{1-q/2}\sqrt{\widehat{\Sigma}_{n,jj}/n}, 
  (\bar{x}_{n})_{j}+z_{1-q/2}\sqrt{\widehat{\Sigma}_{n,jj}/n} \right],
\end{align*}
where $(\bar{x}_{n})_{j}$ represents the $j$-th coordinates of the averaged SGD after $n$ iterations, $\widehat{\Sigma}_{n,jj}$ is the $j$-th diagonal of $\widehat{\Sigma}_n$, and $z_{1-q/2}$ is the $1-q/2$-th percentile of the standard Gaussian distribution.

\section{DISCUSSION}\label{sec:discussion}
In this article, we introduce a novel, fully online de-biased covariance estimator for averaged SGD. The bias-reduction technique significantly improves the estimation accuracy over existing Hessian-free approaches, leading to a convergence rate of $n^{(\alpha-1)/2} \sqrt{\log n}$. Our approach requires substantially less computation than the bootstrap and plug-in estimators while operating entirely on a single SGD trajectory. Beyond its computational efficiency, our method remains applicable in constrained settings where the Hessian or raw data are inaccessible. 

Several promising directions remain for future work. First, a non-asymptotic evaluation of confidence intervals could be obtained via advanced non-asymptotic Gaussian approximations \citep{shao2022berry, wei2025gaussian, sheshukova2025gaussian} when combined with the de-biased covariance estimator. Our refined convergence result would yield a sharper guarantee for coverage rates in finite-sample regimes. Second, since all matrix norms are equivalent for a fixed dimension, our main theorem naturally applies to other matrix metrics, such as the $L_1$ and Frobenius norms. However, explicitly characterizing the dimension-dependent convergence rate under specific matrix norms remains a highly valuable extension. Finally, since our proposed method can be viewed as a de-biased adaptation of the online batch-means estimator, we anticipate that similar generalizations to nonsmooth problems—as recently explored by \cite{Jiang2025Online}—should also apply to the bias-reduction approach.

\subsubsection*{Acknowledgements}
We sincerely thank the program chair, senior area chair, area chair, and the four reviewers for their constructive feedback and involved discussion, which has greatly improved the clarity of our paper. Wei Biao Wu’s research is partially supported by the NSF (Grant NSF/DMS-2311249).
     
	\bibliographystyle{apalike} 	
	\bibliography{reference.bib}

\section*{Checklist}

\begin{enumerate}

  \item For all models and algorithms presented, check if you include:
  \begin{enumerate}
    \item A clear description of the mathematical setting, assumptions, algorithm, and/or model. [Yes]
    \item An analysis of the properties and complexity (time, space, sample size) of any algorithm. [Yes]
    \item (Optional) Anonymized source code, with specification of all dependencies, including external libraries. [Yes]
  \end{enumerate}

  \item For any theoretical claim, check if you include:
  \begin{enumerate}
    \item Statements of the full set of assumptions of all theoretical results. [Yes]
    \item Complete proofs of all theoretical results. [Yes]
    \item Clear explanations of any assumptions. [Yes]     
  \end{enumerate}

  \item For all figures and tables that present empirical results, check if you include:
  \begin{enumerate}
    \item The code, data, and instructions needed to reproduce the main experimental results (either in the supplemental material or as a URL). [Yes]
    \item All the training details (e.g., data splits, hyperparameters, how they were chosen). [Yes]
    \item A clear definition of the specific measure or statistics and error bars (e.g., with respect to the random seed after running experiments multiple times). [Yes]
    \item A description of the computing infrastructure used. (e.g., type of GPUs, internal cluster, or cloud provider). [Yes]
  \end{enumerate}

  \item If you are using existing assets (e.g., code, data, models) or curating/releasing new assets, check if you include:
  \begin{enumerate}
    \item Citations of the creator If your work uses existing assets. [Not Applicable]
    \item The license information of the assets, if applicable. [Not Applicable]
    \item New assets either in the supplemental material or as a URL, if applicable. [Not Applicable]
    \item Information about consent from data providers/curators. [Not Applicable]
    \item Discussion of sensible content if applicable, e.g., personally identifiable information or offensive content. [Not Applicable]
  \end{enumerate}
  We are not using existing assets or releasing new assets.

  \item If you used crowdsourcing or conducted research with human subjects, check if you include:
  \begin{enumerate}
    \item The full text of instructions given to participants and screenshots. [Not Applicable]
    \item Descriptions of potential participant risks, with links to Institutional Review Board (IRB) approvals if applicable. [Not Applicable]
    \item The estimated hourly wage paid to participants and the total amount spent on participant compensation. [Not Applicable]\\
  \end{enumerate}
We did not use crowdsourcing or conduct research with human subjects.
\end{enumerate}


\clearpage
\appendix
\thispagestyle{empty}
\onecolumn
\aistatstitle{Refining Covariance Matrix Estimation in Stochastic Gradient Descent Through Bias Reduction: \\
Supplementary Materials}

The supplement material is organized as follows: In section \ref{sketch}, we demonstrate the high-level idea behind the proof of the main theorem. In section \ref{proof::technical}, we introduce some useful technical lemmas.
In section \ref{proof::iid}, we prove the consistency of our proposed de-biased estimator under the mean estimation model, as well as Proposition \ref{bias} and Theorem \ref{thm:mean}. In section \ref{sec:proof_general} we prove the main theorem in general cases, i.e., Theorem \ref{thm:general}. Additional details and results of numerical experiments are provided in Section \ref{secexp}.

\section{PROOF SKETCH OF THE MAIN THEOREM}\label{sketch}
For a vector $x\in \R^d $, we use $|x|$ to denote its Euclidean norm. For a random variable $X \in \R^d$, we use $\norm{X}_p = (\E\abs{X}^p)^{1/p}$ to denote its $L_p$ norm, and write $\norm{X} = \norm{X}_2$. For a (random) matrix $A \in \R^{d\times d}$, we use $|A|$ or $|A|_2$ to denote its operator norm, use $|A|_F$ to denote its Frobenius norm, and use $\lambda_{max}(A)$,  $\lambda_{min}(A)$ to denote its maximum and minimum eigenvalue. Write $\norm{A}_p = (\E\abs{A}^p)^{1/p}$ and $ \norm{A} = \norm{A}_2$. For a finite set $B$, we use $|B|$ to denote its cardinality.
 
Here and in the sequel, let the SGD iterates take the following form
\begin{equation}
	\begin{split}
		X_i & = X_{i-1} - \eta_i\nabla f(X_{i-1},\xi_i)\\
		& = X_{i-1} - \eta_i\nabla F(X_{i-1}) + \eta_{i}\epsilon_{i}, 
	\end{split}
\end{equation} 
where $\epsilon_{i} =  \nabla F(X_{i-1}) - \nabla f(X_{i-1}, \xi_i)$ is a martingale difference. Define the error sequence 
\[ \delta_{i} = X_{i} - x^{*}\]
which satisfies the recursion
\begin{align}\label{eq3}
	\delta_i= \delta_{i-1} - \eta_i\nabla F(X_{i-1})+\eta_i \epsilon_i.
\end{align}
By expressing $X_i$ in terms of $\delta_i$, the covariance estimator $\widehat\Sigma_n$ can be written as:
\[ \widehat{\Sigma}_{n} = 
		\frac1n  \sum_{i=1}^n \Bigg[  (\delta_i - \bar \delta_n) \bigg(\sum_{k=t_i}^{i} \delta_k - \ell_i\bar \delta_n \bigg)\T +  
		\bigg(\sum_{k=t_i}^{i} \delta_k - \ell_i\bar \delta_n \bigg) (\delta_i - \bar \delta_n)\T
		-  (\delta_i - \bar \delta_n)(\delta_i - \bar \delta_n)\T \Bigg], \] 
where $t_i = i - \ell_i +1$. The goal is to bound the mean squared error:
\[ \MSE(\widehat\Sigma_n) = \| \widehat{\Sigma}_{n}-\Sigma \|^2 \lesssim O(n^{-1+\alpha}\log n).\]

Our strategy is to first approximate the error sequence $\delta_i$ using a linear approximation, and then apply a martingale decomposition involving \textit{i.i.d} approximation.
Specifically, we define
\[U_i = (1 - \eta_i A) U_{i-1} + \eta_i \epsilon_i, U_0 = \delta_0,\] 
as the linear approximation sequence, derived from a first-order Taylor expansion. Letting $r_i = \nabla F(X_{i-1}) - A\delta_{i-1}$ for $i \ge 1$ and $r_0 = 0$, the error due to linear approximation satisfies the recursion 
\[s_i = (1 - \eta_i A) s_{i-1} - \eta_i r_i, s_0 = 0.\]
Similarly, let  $\epsilon_i^* = - \nabla f(x^*, \xi_i)$, which forms a mean-zero i.i.d. sequence. Then, the \textit{i.i.d} approximation sequence is given by
\[\widetilde U_i  = (1 - \eta_i A) \widetilde U_{i-1} + \eta_i \epsilon_i^*, \widetilde U_0 = \delta_0.\]
The error sequence due   to \textit{i.i.d} approximation is 
\[\Delta_i = (1 - \eta_i A)\Delta_{i-1} + \eta_i v_i, \Delta_0 = 0,\]
 where $v_i = \epsilon_i - \epsilon_i^*$.

We next define the estimator without centering:
			\begin{equation}
			\begin{split}
				\widehat{\Sigma}_{n,\delta} = 
				\frac1n  \sum_{i=1}^n \Bigg[  \delta_i \bigg(\sum_{k=t_i}^{i} \delta_k\bigg)\T+\bigg(\sum_{k=t_i}^{i} \delta_k\bigg)\delta_i \T -  \delta_i  \delta_i \T \Bigg],
			\end{split}
		\end{equation}
and the corresponding estimators based on the linear and i.i.d. approximations:
\begin{align*}
		\widehat{\Sigma}_{n, U}& = \frac1n\sum_{i=1}^n \left[ U_iU_i\T + U_i \biggl(\sum_{k=t_i}^{i-1} U_k\biggr)\T+ \biggl(\sum_{k=t_i}^{i-1} U_k\biggr) U_i\T \right],\\
		\widehat{\Sigma}_{n, \wt U}& = \frac1n\sum_{i=1}^n \left[ \widetilde U_i \widetilde U_i\T + \widetilde U_i \biggl(\sum_{k=t_i}^{i-1} \widetilde U_k\biggr)\T+\biggl(\sum_{k=t_i}^{i-1} \widetilde U_k\biggr)\widetilde U_i \T  \right]. \\
\end{align*} 
To bound the estimation error, we decompose it as follows:
\begin{align*}
\|\widehat{\Sigma}_{n}-\Sigma\| \leq   \|\widehat{\Sigma}_{n}-\widehat{\Sigma}_{n,\delta}\|+\|\widehat{\Sigma}_{n,\delta}-\widehat{\Sigma}_{n,U}\|+\|\widehat{\Sigma}_{n,U}-\widehat{\Sigma}_{n,\widetilde{U}}\|
+\|\widehat{\Sigma}_{n,\widetilde{U}}-\Sigma\|
,
\end{align*}

We will use these approximations, both in the mean estimation and general cases, to establish the convergence of our proposed estimator in the subsequent sections.

\section{TECHNICAL SUPPLEMENT}\label{proof::technical}
In this section, we introduce some fundamental technical results, which are frequently applied in the following analysis of SGD iterates. We defer the proof to Section \ref{auxiliary}. We also prove Proposition \ref{onlinebatch} in the paper, which quantifies the batch size.

\begin{lemma} \label{lemma Yij}
	Let $\lambda > 0$ be a constant. For any $i \in \N^+$, define a real sequence $\{Y_{(\lambda)}{}_j^i\}$ with
	\begin{equation}
		Y_{(\lambda)}{}_j^i = \left\{\begin{array}{cc}
			1 & \text{if $j = i$,} \\
			\displaystyle\prod_{k = j+1}^i (1 - \lambda \eta_k) & \text{if $j < i$.}
		\end{array}\right.
	\end{equation}
	Then for all $0 \le j \le i$, we have
	\begin{enumerate}
		\item \label{lemma Yij-1} \begin{equation}
			|Y_{(\lambda)}{}_j^i| \asymp \exp\left\{\frac{\lambda\eta}{1-\alpha} \left(j^{1-\alpha} - i^{1-\alpha} \right)  \right\}
			\le \exp\left\{ - \lambda\eta i^{-\alpha}(i-j) \right\}.
		\end{equation}
         and $ |Y_{(\lambda)}{}_j^i| \asymp \exp\left\{ - \lambda\eta i^{-\alpha}(i-j) \right\}$ if $j\ge i-Ci^{\alpha}\log i$ for some constant $C$.
		\item \label{lemma Yij-2} For $\beta, \gamma > 0$,
		\begin{equation}
			\sum_{j=1}^i \exp\left(\beta j^{1-\alpha}\right) j^{-\gamma\alpha} \asymp \exp\left(\beta i^{1-\alpha}\right) i^{-(\gamma -1)\alpha},
		\end{equation}
		which implies
		\begin{equation}
			\sum_{j=1}^i |Y_{(\lambda)}{}_j^i|^\beta |j^{-\alpha}|^\gamma
			\asymp i^{-(\gamma - 1)\alpha}.
		\end{equation}
		\item \label{lemma Yij-3} For any $n > i$,
		\begin{equation}
			S_{(\lambda)}{}_i^n  :=  \sum_{k = i+1}^n Y_{(\lambda)}{}_i^k \lesssim (i+1)^\alpha.
		\end{equation}
	\end{enumerate}
\end{lemma}

\begin{lemma} \label{lemma batch sum}
	For $\gamma > 0$,
	\begin{equation}
		\sum_{j=t_i}^i j^{-\gamma} \lesssim i^{-\gamma}\ell_i \asymp i^{\alpha-\gamma}\log i.
	\end{equation}
\end{lemma}

\noindent
\begin{lemma}[Rosenthal inequality] \label{lemma Rosenthal}
	If $\xi_1, \cdots, \xi_n$ are independent zero mean random variables and $p \ge 2$, then there exist constant $C_p$ only depends on $p$ such that
	\begin{align}
		\norm{\sum_{i=1}^n \xi_i}_p &\le C_p \max\left\{  \biggl( \sum_{i=1}^n \E\abs{\xi}^2 \biggr)^{1/2},
		\biggl( \sum_{i=1}^n \E\abs{\xi}^p \biggr)^{1/p} \right\} .
	\end{align}
\end{lemma}

\noindent
\begin{lemma}[Burkholder inequality] \label{lemma Burkholder}
	If $D_1, \cdots, D_n$ are martingale differences and $p \ge 2$, then
	\begin{equation}
		\norm{\sum_{i=1}^n D_i}_p^2 \le (p - 1)\sum_{i=1}^n \norm{D_i}_p^2.
	\end{equation}
\end{lemma}

\subsection{Proof of Proposition \ref{onlinebatch}}
\begin{proof}
We consider two cases

(i) Case $i = a_m$: By the definition of the sequence $\{a_m\}$, we have
\[
|B_i| = a_m - a_{m-1} + 1 = \lfloor a_m^{\alpha} \log a_m \rfloor = \lfloor i^{\alpha} \log i \rfloor.
\]
So the claimed bounds hold exactly in this case.
 
(ii) Case $i > a_m$: We write the batch size as
\[
|B_i| = i - a_{m-1} + 1 = (a_m - a_{m-1} + 1) + (i - a_m).
\]
The first term satisfies $a_m - a_{m-1} + 1 = \lfloor a_m^{\alpha} \log a_m \rfloor$ by construction. To analyze the second term, note that:
\[
\lfloor (k+1)^{\alpha} \log(k+1) \rfloor - \lfloor k^{\alpha} \log k \rfloor \le 1, \quad \forall k \ge 1.
\]
By applying a telescoping sum from $k = a_m$ to $i - 1$, we obtain
\[
i - a_m \ge \lfloor i^{\alpha} \log i \rfloor - \lfloor a_m^{\alpha} \log a_m \rfloor.
\]
Therefore,
\[
|B_i| = \lfloor a_m^{\alpha} \log a_m \rfloor + (i - a_m) \ge \lfloor i^{\alpha} \log i \rfloor.
\]
For the upper bound, note that
\[\begin{split}
&i - a_m \\
=& a_{m+1} - a_m - (a_{m+1} - i) \\
\le& \lfloor a_{m+1}^{\alpha} \log a_{m+1} \rfloor - \left( \lfloor a_{m+1}^{\alpha} \log a_{m+1} \rfloor - \lfloor i^{\alpha} \log i \rfloor \right) \\
=& \lfloor i^{\alpha} \log i \rfloor.
\end{split}
\]

This gives
\[
\begin{split}
|B_i| &= \lfloor a_m^{\alpha} \log a_m \rfloor + (i - a_m) \\
&\le \lfloor i^{\alpha} \log i \rfloor + \lfloor i^{\alpha} \log i \rfloor = 2 \lfloor i^{\alpha} \log i \rfloor.
\end{split}
\]

Combining both bounds yields the desired result:
\[
\lfloor i^{\alpha} \log i \rfloor \le |B_i| \le 2 \lfloor i^{\alpha} \log i \rfloor.
\]
\end{proof}

\begin{remark}
In practice, we choose the batch sequence $\tilde{a}_m=Ca_m$ (or the closest integer to $Ca_m$) to obtain some flexibility. Based on Proposition \ref{onlinebatch}, we can show that the batch size $|\tilde{B}_i|$ decided by $\tilde{a}_m$ also meets the condition $|\tilde{B}_i| \asymp i^{\alpha}\log i$: for sufficiently large $i$ and $m$ such that $\tilde{a}_m \leq i < \tilde{a}_{m+1}$, we find an integer $k$ such that $|k-i/C| \leq 1$ and $a_m \leq k < a_{m+1} $. Proposition \ref{onlinebatch} shows $ k-a_{m-1} \asymp k^{\alpha}\log k$, which further implies 
$$i-\tilde{a}_{m-1}  \asymp Ck^{\alpha}\log k \asymp C (\frac{i}{C})^{\alpha} (\log i - \log C) \asymp i^{\alpha} \log i. $$
\end{remark}

\section{MEAN ESTIMATION MODEL}\label{proof::iid}
In this subsection, we consider the data $\{y_i\}_{i \in \N^+}$ generated by a linear model
\[ y_i = x^* + e_i, \]
where $x^* \in \R$ is the true mean to be estimated, and $e_i$'s are i.i.d. random errors with $\E[e_i] = 0$, $\E[e_i^2] < \infty$. Consider the squared loss function at $x$: $f(x, y) = (y - x)^2/2$. Clearly,
\[ F(x) = \E_y[f(x, y)] = \frac12 \E[(x - x^* - e)^2] = \frac{(x - x^*)^2 + \E[e^2]}{2}, \]
and $x^* = \argmin_x F(x)$ holds. Then the $i$-th SGD iterate takes the form
\begin{align*}
	X_i &= X_{i-1} - \eta_i \nabla f(X_{i-1}, y_i) \\
	&= X_{i-1} + \eta_i(y_i - X_{i-1}),
\end{align*}
and the error $\delta_i = X_i - x^*$ takes the form
\begin{align*}
	\delta_i &= \delta_{i-1} + \eta_i(e_i - \delta_{i-1}) \\
	& = (1 - \eta_i)\delta_{i-1} + \eta_i e_i,
\end{align*}
where $\eta_i = \eta i^{-\alpha}$ with $\eta > 0$ and $\alpha \in (0.5, 1)$.

Without loss of generality, assume $x^* = 0$. Then $\delta_i = X_i$ and the SGD iterate takes the form
\begin{equation}\label{eq:SGDform_mean}
	X_i = (1 - \eta_i) X_{i-1} + \eta_i e_i, \quad i \ge 1;
	\qquad X_0 = x_0,
\end{equation}
where $x_0$ is deterministic. We first provide some useful technical lemmas in Section \ref{sec:lemmas_mean_est}. Then, in Section \ref{sec:proof_mean_prop}, we prove Proposition \ref{bias} and Theorem \ref{thm:mean}, where $e_i$ is assumed to be a mean-zero Gaussian random error. Finally, we generalize the results to the case where $e_i$ is non-Gaussian and $\{X_i\}$ is multi-dimensional.

\subsection{Technical Lemmas for the Mean Estimation Model}\label{sec:lemmas_mean_est}
\begin{lemma} \label{indep-X}
	With the definition of $\{Y_{(\lambda)}{}_j^i\}$ in Lemma \ref{lemma Yij}, sequence $\{X_i\}$ can be rewritten as
	\begin{equation}\label{mean_xi}
		X_i =  Y_j^i X_j + \sum_{k=j+1}^i  Y_k^i \eta_k e_k, 
		\quad 0 \le j < i,
	\end{equation}
	where $Y_k^i := Y_{(1)}{}_k^i$. Therefore for all $1\le j < i$, we have 
	\begin{enumerate}
		\item $\var(X_i) \asymp \E[X_i^2] \asymp i^{-\alpha}$.  \label{indep-X-1}
		\item $0 \leq \cov(X_i, X_j) \asymp  \exp\left\{ \frac{\eta}{1 - \alpha} (j^{1-\alpha} - i^{1-\alpha}) \right\} j^{-\alpha} \le \exp\left\{ -\eta i^{-\alpha}(i-j) \right\} j^{-\alpha}$. \label{indep-X-2}
		\item $\cov( X_i, \sum_{k=1}^j X_k ) \asymp  \exp\left\{ \frac{\eta}{1 - \alpha} (j^{1-\alpha} - i^{1-\alpha}) \right\} \le \exp\left\{ -\eta i^{-\alpha}(i-j) \right\}$. \label{indep-X-3}
		\item $\cov( \sum_{k=i}^\infty X_k, X_j) \lesssim  \exp\left\{ \frac{\eta}{1 - \alpha} (j^{1-\alpha} - i^{1-\alpha}) \right\} i^\alpha j^{-\alpha} \le \exp\left\{ -\eta i^{-\alpha}(i-j) \right\} i^\alpha j^{-\alpha}$. \label{indep-X-4}
	\end{enumerate}
\end{lemma}
\begin{proof}
	\vphantom{1}
	\begin{enumerate}
		\item Take $j = 0$ in (\ref{mean_xi}),
		\begin{equation} 
			X_i =  Y_0^i x_0 + \sum_{k=1}^i  Y_k^i \eta_k e_k, \quad i \ge 1.
		\end{equation}
		Recall that $e_k$'s are i.i.d., then by Lemma \ref{lemma Yij}, we have
		\[ \var(X_i) = \E[X_i^2] =  |Y_0^i|^2 x_0^2 + \sum_{k=1}^i |Y_k^i|^2 |\eta_k|^2 \E[e_k^2] \asymp
		 \exp\left\{\frac{\lambda\eta}{1-\alpha} i^{1-\alpha}   \right\}+ i^{-\alpha}
		\asymp i^{-\alpha}. \]
		
		\item Use (\ref{mean_xi}) and apply Lemma \ref{lemma Yij} (\ref{lemma Yij-1}), we have
		\begin{align}
			0 \leq \cov(X_i, X_j) =  Y_j^i \var(X_j)
			& \asymp \exp\left\{ \frac{\eta}{1 - \alpha} (j^{1-\alpha} - i^{1-\alpha}) \right\} j^{-\alpha}
			\nonumber \\
			& \le \exp\left\{ -\eta i^{-\alpha}(i-j) \right\}j^{-\alpha}.  \label{taylor}
		\end{align}

		\item Using result in (\ref{indep-X-2}), and noticing that
		\begin{equation}\label{exact_int}
			\int \exp(\beta u^{1-\alpha}) u^{-\alpha} \dif u = \frac{1}{\beta (1-\alpha)}\exp(\beta u^{1-\alpha}) + C,
		\end{equation}
		we have
		\begin{align*}
			\cov\bigg( X_i, \sum_{k=1}^j X_k \bigg)
			& \asymp \sum_{k=1}^j \exp\left\{ \frac{\eta}{1 - \alpha} (k^{1-\alpha} - i^{1-\alpha}) \right\} k^{-\alpha} \\
			& = \exp\left( -\frac{\eta}{1-\alpha}i^{1-\alpha} \right) \sum_{k=1}^j \exp\left( \frac{\eta}{1 - \alpha}k^{1-\alpha} \right) k^{-\alpha}  \\
			& \asymp \exp\left( -\frac{\eta}{1-\alpha}i^{1-\alpha} \right) \int_1^{j+1} \exp\left( \frac{\eta}{1 - \alpha}u^{1-\alpha} \right) u^{-\alpha} \dif u \tag{eventually increasing in $k$} \\
			& = \exp\left( -\frac{\eta}{1-\alpha}i^{1-\alpha} \right) \eta^{-1}\left[ \exp\left( \frac{\eta}{1 - \alpha} (j+1)^{1-\alpha} \right) - \exp(\beta) \right] \\
			& \asymp \exp\left\{ \frac{\eta}{1 - \alpha} (j^{1-\alpha} - i^{1-\alpha}) \right\}.
		\end{align*}
		The next bound is again obtained by (\ref{taylor}).

		\item Using result in (\ref{indep-X-2}), we have
		\begin{align*}
			\cov\bigg( \sum_{k=i}^\infty X_k, X_j \bigg) & \asymp \sum_{k= i}^\infty \exp\left\{ \frac{\eta}{1 - \alpha} (j^{1-\alpha} - k^{1-\alpha}) \right\} j^{-\alpha} \\
			& = \exp\left( \frac{\eta}{1-\alpha} j^{1-\alpha} \right) j^{-\alpha} \sum_{k= i}^\infty \exp\left( -\frac{\eta}{1 - \alpha}k^{1-\alpha} \right) \\
			& \asymp \exp\left( \frac{\eta}{1-\alpha} j^{1-\alpha} \right) j^{-\alpha} \int_{i-1}^\infty \exp\left( -\frac{\eta}{1 - \alpha}u^{1-\alpha} \right) \dif u \tag{decreasing in $k$} \\
			& \asymp \exp\left( \frac{\eta}{1-\alpha} j^{1-\alpha} \right) j^{-\alpha}
			\int_{\frac{\eta}{1-\alpha} (i-1)^{1-\alpha} }^\infty \e^{-t} t^{\frac{\alpha}{1-\alpha}} \dif t
			\tag{$t = \frac{\eta}{1-\alpha}u^{1-\alpha}$} \\
			& \lesssim \exp\left( \frac{\eta}{1-\alpha} j^{1-\alpha} \right) j^{-\alpha}
			(i-1)^\alpha \exp\left( - \frac{\eta}{1-\alpha} (i-1)^{1-\alpha} \right) \\
			& \asymp \exp\left( \frac{\eta}{1-\alpha} j^{1-\alpha} \right) j^{-\alpha}
			i^\alpha \exp\left( - \frac{\eta}{1-\alpha} i^{1-\alpha} + \mathcal{O}(1) \right) \\
			& \asymp \exp\left\{ \frac{\eta}{1 - \alpha} (j^{1-\alpha} - i^{1-\alpha}) \right\} i^\alpha j^{-\alpha},
		\end{align*}
		where the following bound is used for incomplete gamma function
		\[  \int_a^\infty \e^{-x} x^\beta \dif x \le  a^\beta \e^{-a} C_\beta,
		\tag{$a \ge 1$, $\beta > 1$} \]
		by noticing $\frac{\eta}{1-\alpha} (i-1)^{1-\alpha} \ge 1$ for all $i$ large enough and $\frac{\alpha}{1-\alpha} > 1$. 
		
		The next bound is again obtained by (\ref{taylor}).
	\end{enumerate}

\end{proof}

Recall the debiased estimator (in the univariate case, with sample mean $\bar X_n$ omitted)
\begin{equation}\label{eq:sig_est_1d}
	\widehat{\sigma}_{n} = \frac1n \sum_{i=1}^n \left[ X_i^2 + 2X_i(X_{i-1} + \cdots + X_{t_i}) \right]
	=  \frac1n \sum_{i=1}^n \left[ X_i^2 + 2X_i W_i \right] ,
\end{equation}
with overlapping batch length: $\ell_i = \lfloor C\, i^\alpha \log i \rfloor$, where
\[ W_i  :=  \sum_{j = t_i}^{i-1} X_j.   \tag{$t_i = i - \ell_i + 1$}  \]

\begin{lemma} \label{indep-W}
	For all $1\le j < i$, we have
	\begin{enumerate}
		\item $\var(W_i) \lesssim \ell_i$. \label{indep-W-1}
		\item $\cov(X_i, W_j) \lesssim \exp\left\{ \frac{\eta}{1 - \alpha} (j^{1-\alpha} - i^{1-\alpha}) \right\} 
		\le \exp\left\{ -\eta i^{-\alpha}(i-j) \right\}$. \label{indep-W-2}
		\item $\cov(X_j, W_i) \lesssim \mathcal{O}(1)$. \label{indep-W-3}
		\item $\cov(W_i, W_j) \lesssim \ell_j$. \label{indep-W-4}
	\end{enumerate}
\end{lemma}
\begin{proof}
	\vphantom{1}
	\begin{enumerate}
		\item Note that $\cov(X_i, X_j) \ge 0$ for all $i, j$. Using Lemma \ref{indep-X} (\ref{indep-X-1}, \ref{indep-X-3}), we have
		\begin{align*}
			\var(W_i) & = \sum_{j = t_i}^{i-1} \var(X_j) + 2\sum_{j = t_i + 1}^{i-1}\sum_{k = t_i}^{j-1} \cov(X_j, X_k) \\
			& \le \sum_{j = t_i}^{i-1} \var(X_j) + 2\sum_{j = t_i}^{i-1}\sum_{k = 1}^{j-1} \cov(X_j, X_k) \\
			& \lesssim \sum_{j = t_i}^{i-1} j^{-\alpha} + \sum_{j = t_i}^{i-1} \exp\left\{ \frac{\eta}{1 - \alpha} \left( (j-1)^{1-\alpha} - j^{1-\alpha} \right) \right\}  \\
			& \le \sum_{j = t_i}^{i-1} j^{-\alpha} + \sum_{j = t_i}^{i-1} \exp\left( - \eta j^{-\alpha} \right) \\
			& \lesssim \sum_{j = t_i}^{i-1} \mathcal{O}(1) \lesssim i - t_i = \ell_i - 1 \asymp \ell_i.
		\end{align*}

		\item Using Lemma \ref{indep-X} (\ref{indep-X-3}), we have
		\begin{align*}
			\cov(X_i, W_j) & \le  \sum_{k=1}^{j} \cov(X_i, X_k) \lesssim \exp\left\{ \frac{\eta}{1 - \alpha} (j^{1-\alpha} - i^{1-\alpha}) \right\} \le \exp\left\{ -\eta i^{-\alpha}(i-j) \right\}.
		\end{align*}

		\item Using Lemma \ref{indep-X} (\ref{indep-X-2}, \ref{indep-X-3}, \ref{indep-X-4}), we have
		\begin{align*}
			\cov(X_j, W_i) 
			& \le \sum_{k=1}^\infty \cov(X_j, X_k) \\
			& = \sum_{k=1}^{j-1} \cov(X_j, X_k) + \var(X_j) + \sum_{k=j+1}^\infty \cov(X_j, X_k) \\
			& \lesssim 
			\exp\left\{ \frac{\eta}{1 - \alpha} \left( (j-1)^{1-\alpha} - j^{1-\alpha} \right) \right\} 
			+ j^{-\alpha} \\
			& \quad + \exp\left\{ \frac{\eta}{1 - \alpha} \left( j^{1-\alpha} - (j+1)^{1-\alpha} \right) \right\} \left(\frac{j+1}{j}\right)^\alpha \\
			& \lesssim \exp\left( - \eta j^{-\alpha} \right) + j^{-\alpha} + \exp\left\{ -\eta (j+1)^{-\alpha} \right\} \\
			& = \mathcal{O}(1).
		\end{align*}
		
		

		\item Using result in (\ref{indep-W-3}), we simply have
		\[ \cov(W_i, W_j) = \sum_{k = t_j}^{j - 1} \cov(X_k, W_i)
		\lesssim \sum_{k = t_j}^{j - 1} \mathcal{O}(1)
		\lesssim j - t_j
		= \ell_j - 1 \asymp \ell_j. \]
		
	\end{enumerate}
\end{proof}

In the following Section \ref{sec:proof_mean_prop}, we consider the case where $e_i$'s are Gaussian random variables and assume $x_0 = 0$, same as the setting of Proposition \ref{bias} and Theorem \ref{thm:mean}. Then the true non-asymptotic variance for the mean estimation model is
\begin{equation}\label{eq:sig_finite_1d}
	\sigma_n := \var(\sqrt{n} \bar X_n) = \frac1n \E\abs{\sum_{i=1}^n X_i}^2 
    = \frac1n \sum_{i=1}^n \E\left[ X_i^2 + 2X_i(X_{i-1} + \cdots + X_{1})\right].
\end{equation}



\subsection{Proof of Proposition \ref{bias} and Theorem \ref{thm:mean}}\label{sec:proof_mean_prop}

\begin{proof}[Proof of Proposition \ref{bias}.]
by using Lemma \ref{indep-X} (\ref{indep-X-2}-\ref{indep-X-3}), we get
	\begin{align*}
		0 \leq n \left(\sigma_n - \E[\widehat{\sigma}_{n}]\right) 
		& = \sum_{i=1}^n 2\E \left[ X_i(X_{i-\ell_i} + \cdots + X_1) \right] \\
		& = \sum_{i=1}^n 2 \cdot \sum_{j=1}^{i - \ell_i} \E[X_i X_j]\\
		&=  2 \sum_{i=1}^n \sum_{j=1}^{i - \ell_i} \cov(X_i, X_j)  \\
		& \asymp \sum_{i=1}^n \exp\left\{ \frac{\eta}{1 - \alpha} \left( (i-\ell_i)^{1-\alpha} - i^{1-\alpha}\right) \right\} \\
		& \le \sum_{i=1}^n \exp \left( - \eta i^{-\alpha} \ell_i \right).
	\end{align*}
 Therefore
	\begin{equation}\label{biasbound}
		\abs {\E(\widehat{\sigma}_{n}- \sigma_n )}\lesssim \frac{\sum_{i=1}^n \exp \left( - \eta i^{-\alpha} \ell_i \right) }{n}.
	\end{equation}
    Moreover, if $\ell_i \lesssim i^{\alpha}$, by the same argument of Lemma \ref{lemma Yij} (\ref{lemma Yij-1}), the last inequality becomes $\asymp$ and 
	\begin{equation} 
		\abs {  \E(\widehat{\sigma}_{n}- \sigma_n )} \asymp \frac{\sum_{i=1}^n \exp \left( - \eta i^{-\alpha} \ell_i \right) }{n}.
	\end{equation}
\end{proof}

Before the proof of Theorem \ref{thm:mean}, we present some preliminary results for the moments of normal distribution.

\begin{lemma}[results under normality] \label{gauss}
	Let $X, Y, Z, W$ be joint Gaussian random variables with zero mean, then we have
	\begin{enumerate}
		\item $\cov(XY, ZW) = \cov(X, Z)\cov(Y, W) + \cov(X, W)\cov(Y, Z)$. \label{gauss-1}
		\item $\cov(X^2, Z^2) = 2\cov(X, Z)^2$, $\var(X^2) = 2\var(X)^2$. \label{gauss-2}
		\item $\var(XY) \asymp \var(X)\var(Y)$. \label{gauss-3}.
	\end{enumerate}
\end{lemma}
\begin{proof}
	\vphantom{1}
	\begin{enumerate}
		\item By definition of joint cumulant for zero mean random variables,
		\begin{align*}
			\mathrm{cum}(X, Y, Z, W)
		& = \E[XYZW] - \E[XY]\E[ZW] - \E[XZ]\E[YW] - \E[XW]\E[YZ] \\
		& = \cov(XY, ZW) - \cov(X, Z)\cov(Y, W) - \cov(X, W)\cov(Y, Z).
		\end{align*}
		In addition, the cumulants of degree higher than 2 of Gaussian distribution are zero, which leads to the conclusion.

		\item By setting $X = Y$ and $Z = W$ in (\ref{gauss-1}), we get the first conclusion. The second one is obtained by further letting $X = Z$.

		\item By setting $Z = X$ and $W = Y$ in (\ref{gauss-1}), we get
		\[ \var(XY) = \var(X)\var(Y) + \cov(X, Y)^2. \]
		Note that
		\[ \abs{\cov(X, Y)} \le \sqrt{\var(X)}\sqrt{\var(Y)}, \]
		therefore
		\[ \var(X)\var(Y) \le \var(XY) \le 2 \var(X)\var(Y). \]
	\end{enumerate}
\end{proof}

\begin{theorem}[Gaussian mean estimation model] \label{indep-normal}
	If $e_i \overset{\mathrm{i.i.d.}}{\sim} \mathcal{N}(0, \sigma_e^2)$ and $\ell_i =  \lfloor Ci^{\alpha}\log(i)  \rfloor$ with $\eta C>1$, then
	\begin{equation}
		\mathrm{MSE}\left(\widehat{\sigma}_n\right) := \E\left[ \left(\widehat{\sigma}_{n} - \sigma_n \right)^2  \right]
		\asymp n^{-1+\alpha}\log n.
	\end{equation}
\end{theorem}

\begin{proof}
	Recall the bias-variance decomposition of mean squared error,
	\begin{equation} \label{indep-mse}
		\E\left[ \left(\widehat{\sigma}_{n} - \sigma_n \right)^2  \right]
		= \left(\sigma_n- \E[\widehat{\sigma}_{n}]\right)^2
		+ \var( \widehat{\sigma}_{n} ).
	\end{equation}
	For the bias part, by the same argument as the proof of Proposition 4.1, we have
	\begin{equation} \label{indep-bias-result}
		\abs{\sigma_n - \E[\widehat{\sigma}_{n}]} \lesssim \frac{\sum_{i=1}^n \exp(-\eta C\log i)}{n}\asymp \frac{\sum_{i=1}^n i^{-\eta C}}{n} \asymp n^{-1}.
	\end{equation}

	For the variance part, applying Cauchy-Schwarz inequality $\mathrm{var}(A + B) \le 2\left(\mathrm{var}(A) + \mathrm{var}(B) \right)$, we have 
	\begin{equation} \label{indep-var}
		\var( n \widehat{\sigma}_{n} )
		= \var\bigg( \sum_{i=1}^n \left[ X_i^2 + 2X_i W_i \right] \bigg) 
		\lesssim \var\bigg( \sum_{i=1}^n X_i^2 \bigg) + \var\bigg( \sum_{i=1}^n X_i W_i \bigg).
	\end{equation}

	For the first term, using Lemma \ref{indep-X} (\ref{indep-X-1}, \ref{indep-X-2}), we have
	\begin{align}
		{} & \mathrm{var}\bigg( \sum_{i=1}^n X_i^2 \bigg) \nonumber \\
		= {} &  \sum_{i=1}^n \mathrm{var}(X_i^2) + 2\sum_{i=2}^n \sum_{j=1}^{i-1} \mathrm{cov}(X_i^2, X_j^2) \nonumber \\
		= {} &  2 \sum_{i=1}^n \mathrm{var}(X_i)^2 + 4\sum_{i=2}^n \sum_{j=1}^{i-1} \mathrm{cov}(X_i, X_j)^2
		\tag{under Gaussian} \nonumber \\
		\lesssim {} & \sum_{i=1}^n i^{-2\alpha} + \sum_{i=2}^n \sum_{j=1}^{i-1} \exp\left\{ \frac{2\eta}{1 - \alpha} (j^{1-\alpha} - i^{1-\alpha}) \right\} j^{-2\alpha} \nonumber \\
		\asymp {} &  \mathcal{O}(1) + \sum_{i=2}^n  \exp\left(-\frac{2\eta}{1-\alpha} i^{1-\alpha} \right) \sum_{j=1}^{i-1} \exp\left(\frac{2\eta}{1-\alpha} j^{1-\alpha} \right) j^{-2\alpha} 
		\tag{$\alpha > 1/2$} \nonumber \\
		\lesssim {} & \mathcal{O}(1) + \sum_{i=2}^n  \exp\left(-\frac{2\eta}{1-\alpha} i^{1-\alpha} \right) \int_{0}^{i} \exp\left(\frac{2\eta}{1-\alpha} u^{1-\alpha} \right) u^{-\alpha}  \dif u 
		\tag{drop a $j^{-\alpha}$} \nonumber \\
		= {} & \mathcal{O}(1) + \sum_{i=2}^n \exp\left(-\frac{2\eta}{1-\alpha} i^{1-\alpha} \right) (2\eta)^{-1} \left[ \exp\left(\frac{2\eta}{1-\alpha} i^{1-\alpha} \right) - 1 \right]
		 \tag{by (\ref{exact_int})} \nonumber \\
		\asymp {} & \mathcal{O}(1) + \sum_{i=2}^n \mathcal{O}(1)
		\asymp n.  \label{indep-var-I}
	\end{align}
	For the second term,
	\[
		\mathrm{var}\bigg( \sum_{i=1}^n X_i W_i \bigg)
		= \sum_{i=1}^n \mathrm{var}(X_i W_i) + 2\sum_{i=2}^n \sum_{j=1}^{i-1} \mathrm{cov}(X_i W_i, X_j W_j) := I + I\!I,
	\]
	where, using Lemma \ref{indep-X} (\ref{indep-X-1}) and Lemma \ref{indep-W} (\ref{indep-W-1}), we have
	\begin{align*} 
		I & \asymp \sum_{i=1}^n \mathrm{var}(X_i)\mathrm{var}(W_i)  \tag{under Gaussian} \\
		& \lesssim \sum_{i=1}^n i^{-\alpha} \cdot  \ell_i  \\
		& \asymp \sum_{i=1}^n \log i \asymp n \log n, \tag{$\ell_i = \lfloor C i^\alpha \log i \rfloor$}
	\end{align*}
	and (under Gaussian)
		\begin{align*}
			I\!I &\asymp \sum_{i=2}^n \sum_{j=1}^{i-1} \mathrm{cov}(X_i, X_j)\mathrm{cov}(W_i, W_j)
			+ \sum_{i=2}^n \sum_{j=1}^{i-1} \mathrm{cov}(X_i, W_j)\mathrm{cov}(X_j, W_i) := I\!I\!A + I\!I\!B.
		\end{align*}
	For the first part, using Lemma \ref{indep-X} (\ref{indep-X-1}) and Lemma \ref{indep-W} (\ref{indep-W-4}),
		\begin{align}
			I\!I\!A & \lesssim
			\sum_{i=1}^n \sum_{j=2}^{i-1} 
			\exp\left\{ \frac{\eta}{1 - \alpha} (j^{1-\alpha} - i^{1-\alpha}) \right\} j^{-\alpha} \ell_j  \label{indep-var-IIA} \\
			& \lesssim \sum_{i=1}^n \sum_{j=2}^{i-1} \exp\left\{ -\eta i^{-\alpha} (i - j) \right\} \log j 
			\tag{$\ell_j = \lfloor C j^\alpha \log j \rfloor$} \nonumber \\
			& \le \sum_{i=1}^n \log i \sum_{k=1}^{\infty} \exp\left( -\eta i^{-\alpha} k \right) \tag{$k = i - j$} \nonumber \\
			& = \sum_{i=1}^n \log i  \cdot \frac{\exp\left( -\eta i^{-\alpha} \right)}{ 1 - \exp\left( -\eta i^{-\alpha} \right)} \nonumber \\
			& = \sum_{i=1}^n \log i  \left( \exp\left(\eta i^{-\alpha} \right) - 1 \right)^{-1} \nonumber \\
			& \le \sum_{i=1}^n \eta^{-1} i^\alpha\log i  \tag{$\exp(x) - 1 \ge x$} \nonumber \\ 
			& \asymp  n^{1+\alpha} \log n. \nonumber
		\end{align}
		Similarly, for the second part, using Lemma \ref{indep-W} (\ref{indep-W-2}, \ref{indep-W-3}),
		\begin{equation*}
			I\!I\!B \lesssim \sum_{i=1}^n \sum_{j=2}^{i-1} 
			\exp\left\{ \frac{\eta}{1 - \alpha} (j^{1-\alpha} - i^{1-\alpha}) \right\},
		\end{equation*}
		which is bounded by (\ref{indep-var-IIA}) in $I\!I\!A$. So $I\!I \asymp I\!I\!A + I\!I\!B \lesssim n^{1+\alpha} \log n$, and therefore
		\begin{equation}
			\mathrm{var}\bigg( \sum_{i=1}^n X_i W_i \bigg) = I + I\!I \lesssim n \log n + n^{1+\alpha} \log n
		\asymp n^{1+\alpha} \log n. \label{indep-var-II}
		\end{equation}
		Substitute (\ref{indep-var-I}, \ref{indep-var-II}) into (\ref{indep-var}), we have
		\begin{equation*}
			\var( n \widehat{\sigma}_{n} )
		\lesssim \var\bigg( \sum_{i=1}^n X_i^2 \bigg) + \var\bigg( \sum_{i=1}^n X_i W_i \bigg) \lesssim n + n^{1+\alpha} \log n \asymp n^{1+\alpha} \log n.
		\end{equation*}
		That is,
		\begin{equation} \label{indep-var-result}
			\var( \widehat{\sigma}_{n} )\lesssim n^{-1+\alpha}\log n.
		\end{equation}
        For the lower bound, it suffices to show that 
        $$ \var (\sum_{i=1}^n X_iW_i ) \gtrsim n^{\alpha+1}\log n.$$
	To this end, we use the fact that all covariances $\cov (X_iX_j) \ge 0$. Hence
    \begin{align*}
  \var (\sum_{i=1}^n X_iW_i ) &=\sum_{i=1}^n\sum_{j=1}^n \cov (X_iW_i,X_jW_j)\\
  &=\sum_{i=1}^n\sum_{j=1}^n [\cov (X_i,X_j) \cov (W_i,W_j)+\cov (X_i,W_j) \cov (W_i,X_j)]\\
  &\ge\sum_{i=1}^n\sum_{j=1}^n \cov (X_i,X_j) \cov (W_i,W_j)\\
  &\ge \sum_{i=\mathcal{B} n}^{n-\mathcal{B}n} \ \sum_{j=i}^{i+\tau n^{\alpha}\log n} \cov (X_i,X_j) \cov (W_i,W_j)
    \end{align*} 
for some constant $0<\tau<\mathcal{B}<1/2$, and $\tau<C\mathcal{B}^{\alpha}$. Notice that for $\mathcal{B}n\leq i\leq j \leq i+\tau n^{\alpha} \log n \leq n$,  the set $\{ X_{t_i},...,X_i\}$ and $\{ X_{t_j},...,X_j\}$ have at least 
$$Ci^{\alpha}\log i-\tau n^\alpha \log n\ge (C\mathcal{B}^{\alpha}-\tau) n^{\alpha}\log n +C\mathcal{B}^{\alpha}n^{\alpha}\log \mathcal{B}\gtrsim \widetilde{C}n^{\alpha}\log n.$$ 
consecutive common terms for some constant $\widetilde{C}$, and at most $Cn^{\alpha}\log n$ common terms. Denote $\mathcal{A}_{ij}$ as the intersection of $\{ X_{t_i},...,X_i\}$ and $\{ X_{t_j},...,X_j\}$. When $\mathcal{B}n\leq i\leq j \leq i+\tau n^{\alpha} \log n \leq n$, we have
    \begin{equation}\label{covlowerbound}
     \cov(W_i,W_j)\ge \sum_{k,l \in \mathcal{A}_{ij}} \Cov(X_k,X_l) =\sum_{k\in \mathcal{A}_{ij}} \sum_{l \in \mathcal{A}_{ij}}\cov(X_k,X_l).
    \end{equation} 
Then we claim that for any batch $\{X_t, ...,X_{t+m}\}$ with $t \asymp n$ and $m \asymp n^{\alpha}\log n$,
$$ \sum_{i=t}^{t+m} \cov(X_t,X_i)\gtrsim 1.$$
The proof of this claim leverages Lemma \ref{indep-X} (\ref{indep-X-1}) and the argument of Lemma \ref{lemma Yij} (\ref{lemma Yij-1}):

    \begin{align*}
    \sum_{i=t}^{t+m} \cov(X_t,X_i) &\asymp \sum_{i=t}^{t+m} t^{-\alpha}\exp\{ \frac{\eta}{1 - \alpha} (t^{1-\alpha} - i^{1-\alpha}) \}\\
    &=    \sum_{i=t}^{t+m}t^{-\alpha} \exp\left\{ -\eta t^{-\alpha}(i-t) + \frac12 \eta\alpha u^{-\alpha - 1}(i-t)^2 \right\}
		\tag{for some $u \in (t, i)$} \\
        &\asymp  \sum_{i=t}^{t+m} t^{-\alpha}\exp\left\{ -\eta t^{-\alpha}(i-t)  \right\} \tag{$u^{-\alpha - 1}(i-t)^2 \rightarrow 0$}\\
        & = t^{-\alpha}\frac{1-\exp(-\eta t^{-\alpha}(m+1))}{1-\exp(-\eta t^{-\alpha})} \gtrsim1.
    \end{align*} 
The claim and \eqref{covlowerbound} immediately imply $\cov(W_i,W_j) \gtrsim n^{\alpha}\log n$ when $\mathcal{B}n\leq i\leq j \leq i+\tau n^{\alpha} \log n \leq n$. Moreover, we have
\begin{align*}
  \var (\sum_{i=1}^n X_iW_i ) &\ge \sum_{i=\mathcal{B} n}^{n-\mathcal{B}n} \ \sum_{j=i}^{i+\tau n^{\alpha}\log n} \cov (X_i,X_j) \cov (W_i,W_j)\\
  &\gtrsim \sum_{i=\mathcal{B} n}^{n-\mathcal{B}n} \ \sum_{j=i}^{i+\tau n^{\alpha}\log n}\cov (X_i,X_j)n^{\alpha}\log n \\
  &\gtrsim n^{\alpha}\log n\sum_{i=\mathcal{B} n}^{n-\mathcal{B}n} 1 \ge (1-2\mathcal{B})n^{1+\alpha}\log(n).
    \end{align*} 
This completes the proof of the lower bound, and we have $\var(\widehat{\sigma}_n) \asymp n^{-1+\alpha}\log n$. Finally, by (\ref{indep-mse}), we get
		\begin{align*}
			\E\left[ \left(\widehat{\sigma}_{n} - \sigma_n \right)^2  \right]
			 = \left(\sigma_n - \E[\widehat{\sigma}_{n}]\right)^2
			+ \var( \widehat{\sigma}_{n} ) \asymp n^{-1+\alpha}\log n.
		\end{align*}

\end{proof}

Now we are ready to prove Theorem \ref{thm:mean} by a slight modification of the proof of Theorem \ref{indep-normal}. 

\begin{proof}[Proof of Theorem \ref{thm:mean}]
Notice that the only difference between Theorem \ref{thm:mean} and Theorem \ref{indep-normal} is the target variance, $\sigma$ and $\sigma_n$. The variance of the estimator remains unchanged, and we have obtained that $\var(\widehat{\sigma}_n) \asymp n^{-1+\alpha}\log n$. For the bias part, in the proof of Proposition \ref{indep-general} (the fourth term in \eqref{eq:decomp}), we show that $\abs{\sigma_n - \sigma  } \lesssim n^{-1+\alpha}$. Together with \eqref{indep-bias-result} and triangle inequality, we have $\abs{\E\widehat{\sigma}_n - \sigma    } \lesssim n^{-1+\alpha} $. Finally, the bias-variance decomposition implies
\begin{equation} 
		\E\left[ \left(\widehat{\sigma}_{n} - \sigma \right)^2  \right]
		= \left(\sigma- \E[\widehat{\sigma}_{n}]\right)^2
		+ \var( \widehat{\sigma}_{n} ) \asymp n^{-1+\alpha}\log n.
	\end{equation}
\end{proof}

\subsection{Generalization to Non-Gaussian Error Case}
Theorem \ref{indep-normal} can be generalized to the non-Gaussian error case by the following result.
\begin{lemma}[normal comparison]  \label{indep-normal-comp}
    Consider the following two SGD iterates 
    \begin{align*}
	X_i = (1 - \eta_i) X_{i-1} + \eta_i e_i, \quad i \ge 1;
	\qquad X_0 = 0, \\
        \wt X_i = (1 - \eta_i) \wt X_{i-1} + \eta_i \wt e_i, \quad i \ge 1;
	\qquad \wt X_0 = 0,
    \end{align*}
    with $e_i \iid \Pi$, $\E[e_i] = 0$, $\E[e_i^2] = \sigma_e^2$, $\E[e_i^4] < \infty$, $\wt e_i \iid \mathcal{N}(0, \sigma_e^2)$. The corresponding estimators are
    \begin{align*}
	\widehat{\sigma}_n = \frac1n \sum_{i=1}^n \left[ X_i^2 + 2X_i(X_{i-1} + \cdots + X_{t_i}) \right], \\
        \wt{\sigma}_n = \frac1n \sum_{i=1}^n \left[ \wt X_i^2 + 2 \wt X_i( \wt X_{i-1} + \cdots + \wt X_{t_i}) \right].
    \end{align*}
    Then there exists constants $0 \le C_1 \le C_2 < \infty$, such that $C_1 \var(\wt{\sigma}_n) \le \var(\widehat{\sigma}_n) \le C_2 \var(\wt{\sigma}_n)$.
    
\end{lemma}

\begin{proof}
Note that $\wh\sigma_{n}$ can be expressed in the following form 
 $$
 \widehat{\sigma}_n = \sum_{i=1}^{n}\sum_{j=1}^{n}\beta_{i,j} e_{i} e_{j} = \sum_{i=1}^{n}\beta_{i,i} e_{i}^{2} +\sum_{1 \le i < j \le n}\beta_{i,j} e_{i}e_{j}.
 $$
It's easy to check that all $n + \binom{n}{2}$ terms in the above summation are uncorrelated with each other.  
Then we have
\begin{align*}
    \var( \wh\sigma_{n} ) & = \sum_{i=1}^{n}\beta_{i,i}^{2} \var(e_{i}^{2}) + \sum_{1 \le i < j \le n}\beta_{i,j}^{2} \var(e_{i} e_{j}) \\
    & = \sum_{i=1}^{n}\beta_{i,i}^{2} \cdot (\kappa - 1)\sigma_e^4 + \sum_{1 \le i < j \le n}\beta_{i,j}^{2} \sigma_e^4.
\end{align*}
where $\E[e_{i}^4] = \kappa \sigma_e^4$ for some $1 \le \kappa < \infty$. Similarly,
\[ \var( \wt\sigma^2_{n} ) 
    = \sum_{i=1}^{n}\beta_{i,i}^{2} \cdot 2 \sigma_e^4 + \sum_{1 \le i < j \le n}\beta_{i,j}^{2} \sigma_e^4.  \]
As long as $0 \le C_1 \le \frac{\kappa - 1}{2} \le C_2 < \infty$, we have $C_1 \var(\wt{\sigma}_n) \le \var(\widehat{\sigma}_n) \le C_2 \var(\wt{\sigma}_n)$.
\end{proof}

Note that the bound for the bias term in the proof of Theorem \ref{indep-normal} is invariant under the distribution of $e_i$. Therefore to derive the MSE of the de-biased estimator beyond Gaussian noise, it suffices to bound the variance term, which is shown to be the same order as that in Gaussian noise cases by Lemma \ref{indep-normal-comp}. Consequently, we have the following conclusion for the general mean estimation model.
\begin{corollary}[general mean estimation model] \label{indep-non-gauss}
	Suppose $x_0 = 0$, $e_i \iid \Pi$ with $\E[e_i] = 0$ and $\E[e_i^4] < \infty$. Then
	\begin{equation}
		\mathrm{MSE}\left(\widehat{\sigma}_{n}\right) := \E\left[ \left(\widehat{\sigma}_{n} - \sigma_n \right)^2  \right]
		\lesssim n^{-1+\alpha}\log n.
	\end{equation}
\end{corollary}

\begin{remark}
The conclusion also holds if the SGD has the iteration
$$X_i = (1- \eta_iA)X_{i-1}+\eta_ie_i=(1- \eta_iA)X_{i-1}+\eta_iA (A^{-1}e_i).$$
for some $A>0$ and $e_i \iid \Pi$, $\E[e_i] = 0$, $\E[e_i^4] < \infty$. Because we can treat $\eta_i A$ as the new step size, and $A^{-1}e_i$ as the new \textit{i.i.d} noise term. Then the setting is exactly the same as the setting of Corollary \ref{indep-non-gauss}. 
\end{remark}

\subsection{Generalize to Multi-dimension}
We present the above results and proofs in the one-dimensional case. However, the generalization to the multi-dimensional setting is natural. Before the proof of the multivariate case, we first introduce an ordering Lemma that helps with the error decomposition of random matrices.
\begin{lemma}\label{matrixv}
    Denote $V(M)= \E \| \E M -M\|^2_F$, the sum of variances of all entries of a $d \times d$ random matrix $M$. Similarly, define $U(M_1,M_2)$ to be the sum of the covariances of all corresponding entries of two $d \times d$ random matrices $M_1$ and $M_2$. For constant positive-definite matrices $P_1,P_2,P_3,P_4$ and a random matrix $M$, we have
$$ |U(P_1MP_2, P_3MP_4)| \leq \|P_1\|_2\|P_2\|_2\|P_3\|_2\|P_4\|_2 V(M),$$
and when $P_1,P_2,P_3,P_4$ are all scalar matrices, the equality holds.
\end{lemma}
\begin{proof}
    Define $\text{vec}(M)$ to be the vectorization of a matrix $M$, i.e., it stacks the columns of $M$ into a single, long column vector. Let $\Sigma = \text{Cov}(\text{vec}(M))$ be the covariance matrix of $\text{vec}(M)$ which is positive semi-definite. Then we have
$$\operatorname{vec}(P_1 M P_2)
=(P_2^{\top}\otimes P_1)\,\operatorname{vec}(M).$$
Let $\text{vec}(M)= v $, $B=P_2^{\top}\otimes P_1$ and $C=P_4^{\top}\otimes P_3$, then $\operatorname{vec}(P_1MP_2)=Bv$ and $\operatorname{vec}(P_3MP_4)=Cv$. Thereby
\[
U(P_1MP_2,P_3MP_4)
=\operatorname{trace}(\operatorname{Cov}(Bv,Cv))
=\operatorname{trace}\big( B\Sigma C^{\top}\big)=\operatorname{trace}\big( C^{\top}B\Sigma \big),
\]
and similarly $\operatorname{trace}(\Sigma) =V(M)$. For a positive semi-definite $\Sigma$, by its spectral decomposition, $\Sigma = \sum_{k=1}^{d^2} \lambda_k u_k u_k^{\top}$ where $\lambda_k \ge 0$. As a result,
$$ |\operatorname{trace}\big( C^{\top}B\Sigma \big)|= |\sum_{k=1}^{d^2}\operatorname{trace}(\lambda_k  C^{\top}B u_ku_k^{\top})| =|\sum_{k=1}^{d^2}\lambda_k  u_k^{\top}C^{\top}B u_k| \le \sum_{k=1}^{d^2}\lambda_k \|C^{\top}B\|_2 \|u_k\|_2^2 =  \|C^{\top}B \|_2 \operatorname{trace}(\Sigma), $$
Then the conclusion is implied by
$$\|C^{\top}B \|_2\leq \|C \|_2  \|B\|_2 = \|P_1\|_2\|P_2\|_2\|P_3\|_2\|P_4\|_2,$$
due to the property of the Kronecker product: $\|P_2^{\top}\otimes P_1\|_2 = \| P_1\|_2 \| P_2 \|_2$ and $\|P_4^{\top}\otimes P_3\|_2 = \| P_3\|_2 \| P_4 \|_2$. When $P_1,P_2,P_3,P_4$ are all scalar matrices, the equality clearly holds.\\
\end{proof}
The high-level idea to prove the multivariate convergence rate is to control the norm of a matrix by its quadratic form, which is stated in the following lemma.
\begin{lemma}\label{fnormbound}
Let $u_i \in \R^d$ be the unit vector with the $i$-th coordinate equal $1$, and for $i \neq j$ let 
$$u_{ij+}=\frac{u_i+u_j}{\sqrt{2}}, \,\ u_{ij-}=\frac{u_i-u_j}{\sqrt{2}}.$$
Then for any real symmetric $d\times d$ matrix $M$,
$$\| M\|_F^2 \leq \sum_{i=1}^d (u_i \T Mu_i)^2 +  \sum_{1\leq i<j\leq d}(u_{ij+}\T M u_{ij+})^2 +\sum_{1\leq i<j\leq d}(u_{ij-}\T M u_{ij-})^2 . $$
\end{lemma}
\begin{proof}
Write $M = (M_{ij})_{1\leq i,j\leq d}$. It is clear that $M_{ii} =  u_i\T Mu_i$. Further notice that 
$$(u_{ij+}\T M u_{ij+})^2=\frac{1}{4}(M_{ii}+M_{jj}+2M_{ij})^2,\,\ (u_{ij-}\T M u_{ij-})^2=\frac{1}{4}(M_{ii}+M_{jj}-2M_{ij})^2,$$
we have
$$ (u_{ij+}\T M u_{ij+})^2+(u_{ij-}\T M u_{ij-})^2 = \frac{1}{2}(M_{ii}+M_{jj})^2+ 2(M_{ij})^2\ge   2(M_{ij})^2 ,$$
and the conclusion follows.
\end{proof}

Now we generalize the results to the multivariate case and show that
$$\E\| \wh\Sigma_n  - \Sigma \|_F^2  \lesssim  n^{-1+\alpha} \log n, \,\  \text{where} \,\ \widehat{\Sigma}_n = \frac1n \sum_{i=1}^n \left[ X_iX_i\T + X_i(X_{i-1} + \cdots + X_{t_i})\T+(X_{i-1} + \cdots + X_{t_i})X_i\T \right].$$
We focus on the Frobenius norm of the error, which is larger than the operator norm. First, consider the multivariate iteration $X_i = (1-\eta_i) X_{i-1} + \eta_i e_i$, i.e., $X_i\in\R^{d}$ follows the same form of iteration as in \eqref{eq:SGDform_mean}. For any $a \in \R^d$ with $\|a\| = 1$, define $Y_i = a\T X_i$ . Then
\begin{equation}\label{newite}
    Y_i = a\T X_i = (1-\eta_i) a\T X_{i-1} + \eta_i a\T e_i   = (1-\eta_i) Y_{i-1} + \eta_i \wt e_i,\qquad  Y_0 = 0
\end{equation}  
where $\wt e_i = a\T e_i$, $\text{var}(\wt e_i) \le \lambda_{\max}(\Sigma_e)$ and $\E |\wt e_i|^4 \le \E |e_i|^4$. 

 Let $\wh\sigma^2_{n, Y}, \sigma^2_{n, Y}, \sigma_{Y}$ denote, respectively, the  covariance estimator (as in \eqref{eq:sig_finite_1d}), the true non-asymptotic variance (as in \eqref{eq:sig_est_1d}) and the limiting variance (sandwich form) associated with the  sequence $\{Y_i\}$. Similarly, define $\wh\Sigma_n$ , $\Sigma_n$, $\Sigma$ as as the corresponding quantities in the multi-dimensional setting, based on the sequence $\{X_i\}$. Then  
 \[a\T \wh\Sigma_n a = \wh\sigma^2_{n, Y}, 
 a\T \Sigma_n a  = \sigma^2_{n, Y},
 a\T \Sigma a= \sigma_{Y}.\]
Since $\{Y_i\}$ is a one-dimensional SGD iterate of the same form as defined in \eqref{eq:SGDform_mean}, we can apply the one-dimensional bound:
\[  \E\Big[ (\wh\sigma^2_{n, Y}  -  \sigma^2_{n, Y} )^2 \Big]\lesssim n^{-1+\alpha} \log n,\]
or
\[  \E\Big[ (\wh\sigma^2_{n, Y}  -  \sigma^2_{Y} )^2 \Big]\lesssim n^{-1+\alpha} \log n.\]
In other words, for any $a \in \R^d$ with $\|a\| = 1$, since the noise term $\wt e_i $ the iteration \eqref{newite} has uniformly bounded variances and $4$-th moments, we have
\begin{equation}\label{quadraticbound}
\E\Big[ a\T (\wh\Sigma_n -\Sigma_n) a \Big]^2 \lesssim n^{-1+\alpha} \log n.    
\end{equation}
By Lemma \ref{fnormbound}, we can choose $a$ as $u_i$, $u_{ij+}$ and $u_{ij-}$ for $1\leq i <j \leq d$, and the Frobenius norm can be bounded by a finite sum of such quadratic forms. As a result, \eqref{quadraticbound} directly implies
\[\E\| \wh\Sigma_n  - \Sigma_n \|_F^2  \lesssim  n^{-1+\alpha} \log n. \]
Similarly we have
\[\E\| \wh\Sigma_n  - \Sigma \|_F^2  \lesssim  n^{-1+\alpha} \log n. \]
Finally, we prove the convergence rate for the iteration with an arbitrary positive-definite matrix $A$, i.e.,
$$X_i = ( \textbf{I}_d- \eta_iA) X_{i-1} + \eta_i e_i, \qquad  X_0 = 0.$$
The high-level idea is to apply Lemma \ref{matrixv} to get an upper bound of $V(\widehat{\Sigma}_n)$ as a linear combination of $V(e_ie^{\top}_j)$, and show that this upper bound reduces to the case that has been discussed and solved. We decompose the estimation error into the bias and variance parts,
$$ \E \| \wh \Sigma_n - \Sigma_n \|^2_F = \| \E \wh \Sigma_n - \Sigma_n    \|^2_F + \E \| \E \wh \Sigma_n-\wh \Sigma_n\|^2_F.$$

The first term, i.e., the bias part, can be bounded following the same argument in the proof of Proposition \ref{bias} (Section \ref{sec:proof_mean_prop}). For the variance part, we define another SGD iteration in $\R^d$,
$$ \wt X_i = (1 - \eta_i \lambda) \wt X_{i-1} + \eta_i  e_i, 
	\qquad \wt X_0 = 0,$$
where $\lambda$ is the smallest eigenvalue of $A$. We denote the corresponding covariance estimator as $\wt \Sigma_n$. Notice that we have shown the convergence rate for $\wt \Sigma_n$, i.e., we have $\E \| \E \wt \Sigma_n-\wt \Sigma_n\|^2_F \lesssim n^{-1+\alpha} \log n$. Then it suffices to show that 
\begin{equation}
   \E \| \E \wh \Sigma_n-\wh \Sigma_n\|^2_F \leq\E \| \E \wt \Sigma_n-\wt \Sigma_n\|^2_F,
\end{equation}
i.e., $ V (\wh \Sigma_n) \leq V( \wt \Sigma_n) $. Recall that $X_i = \sum_{k=1}^i Y_{(A)}{}_k^i\eta_k e_k $, similar as the proof of \ref{indep-normal-comp}, we express $\wh \Sigma_n$ as $n + \binom{n}{2}$ uncorrelated summation:
 $$
 \widehat{\Sigma}_n = \sum_{i=1}^{n}\sum_{j=1}^{n}h_{i,j}( e_{i} e_{j}\T) = \sum_{i=1}^{n} h_{i,i}( e_{i}e_i\T) +\sum_{1 \le i < j \le n}h_{i,j} (e_{i}e_{j}\T),
 $$
where $h_{i,j}$ are linear functions of $d \times d$ matrices such that for some finite index sets $\{ S_{i,j} \in \{1, 2, \ldots, n\} \times \{1, 2, \ldots, n\}
, 1 \leq i \leq j\leq n\}$,
$$h_{i,j} (e_{i}e_{j}\T) =\eta_i\eta_j \sum_{t,k \in \mathcal{S}_{i,j}} Y_{(A)}{}_i^t e_{i}e_{j}\T Y_{(A)}{}_j^k.$$
Here we define 
\begin{equation}
		Y_{(A)}{}_j^i = \left\{\begin{array}{cc}
			\textbf{I}_d & \text{if $j = i$,} \\
			\displaystyle\prod_{k = j+1}^i (\textbf{I}_d -  \eta_kA) & \text{if $j < i$.}
		\end{array}\right.
	\end{equation}

Leveraging Lemma \ref{matrixv}, we have an upper bound of $V(\wh \Sigma_n)$ as
$$V(\wh \Sigma_n) \leq  \sum_{i=1}^{n} \eta_i^2 V( e_{i}e_i\T)(\sum_{t,k \in \mathcal{S}_{i,i}}Y_{(\lambda)}{}_i^t Y_{(\lambda)}{}_i^k)^2+\sum_{1 \le i < j \le n}\eta_i\eta_jV (e_{i}e_{j}\T)(\sum_{t,k \in \mathcal{S}_{i,j}}Y_{(\lambda)}{}_i^t Y_{(\lambda)}{}_j^k)^2.$$
The key observation is that the structures of $\wt \Sigma_n$ and $\wh \Sigma_n$ are exactly the same, except that the matrix $A$ should be replaced by $\lambda \textbf{I}_d$. Following an identical argument,
$$V(\wt \Sigma_n) =  \sum_{i=1}^{n} \eta_i^2 V( e_{i}e_i\T)(\sum_{t,k \in \mathcal{S}_{i,i}}Y_{(\lambda)}{}_i^t Y_{(\lambda)}{}_i^k)^2+\sum_{1 \le i < j \le n}\eta_i\eta_jV (e_{i}e_{j}\T)(\sum_{t,k \in \mathcal{S}_{i,j}}Y_{(\lambda)}{}_i^t Y_{(\lambda)}{}_j^k)^2$$
since by Lemma \ref{matrixv}, the equality holds for scalar multiplier matrices. So we have proved that $V(\wh \Sigma_n) \leq V(\wt \Sigma_n) \lesssim n^{-1+\alpha} \log n$, and the proof for multivariate case is completed.

 In the remainder of the proof, we will use the one-dimensional results established in this section, while all results remain valid in the multi-dimensional setting.

\section{PROOF OF THEOREM \ref{thm:general}} \label{sec:proof_general}
As mentioned in the proof sketch, to bound the estimation error, we decompose it as follows:
\begin{equation}\label{eq:decomp}
    \|\widehat{\Sigma}_{n}-\Sigma\| \leq   \|\widehat{\Sigma}_{n}-\widehat{\Sigma}_{n,\delta}\|+\|\widehat{\Sigma}_{n,\delta}-\widehat{\Sigma}_{n,U}\|+\|\widehat{\Sigma}_{n,U}-\widehat{\Sigma}_{n,\widetilde{U}}\|+\|\widehat{\Sigma}_{n,\widetilde{U}}- \Sigma\|.
\end{equation}

 \subsection{With i.i.d Approximation: the Fourth Term}
Recall the definition of \textit{i.i.d} approximation sequence,
\[\widetilde U_i  = (1 - \eta_i A) \widetilde U_{i-1} + \eta_i \epsilon_i^*; \quad \widetilde U_0 = \delta_0,\]
where $\epsilon_i^* = - \nabla f(x^*, \xi_i)$. The corresponding estimator based on this sequence is given by
\[\widehat{\Sigma}_{n, \wt U}= \frac1n\sum_{i=1}^n \left[ \widetilde U_i^2 + 2\widetilde U_i \biggl(\sum_{k=t_i}^{i-1} \widetilde U_k\biggr)  \right] \]
We consider the starting point $\delta_0$ to be deterministic. The next proposition bound the fourth term in \eqref{eq:decomp}.
\begin{proposition} \label{indep-general}
	Assume $\E[ \nabla f(x^*, \xi_i)^4] < \infty$, then
	\begin{equation}
		\| \widehat{\Sigma}_{n, \wt U}- \Sigma  \|_2 = \sqrt{ \E\abs{\widehat{\Sigma}_{n, \wt U}- \Sigma }^2  }
		\lesssim \sqrt{n^{-1+\alpha}\log n}.
	\end{equation}
\end{proposition}
\begin{proof}
        Define a sequence with starting point $0$,
        \begin{equation}
	X_i^\circ = (1 - \eta_iA) X_{i-1}^\circ + \eta_i \epsilon_i^*, \quad i \ge 1;
	\qquad X_0^\circ = 0,
\end{equation}
    and the corresponding estimator
    \begin{equation}
	\widehat{\Sigma}^{\circ}_{n} = \frac1n \sum_{i=1}^n \left[ |X_i^\circ|^2 + 2X_i^\circ(X_{i-1}^\circ + \cdots + X_{t_i}^\circ) \right].
\end{equation}
	Decompose
	\begin{equation}
		\| \widehat{\Sigma}_{n, \wt U}- \Sigma \|_2
		\le \| \widehat{\Sigma}_{n, \wt U}- \widehat{\Sigma}^{\circ}_{n} \|_2 + \| \widehat{\Sigma}^{\circ}_{n} - \Sigma^\circ_n \|_2 + \| \Sigma^\circ_n - \Sigma_{n, \wt U} \|_2+ \|\Sigma_{n, \wt U}-\Sigma\|_2.
	\end{equation}
	The bound for the second term is given by Corollary \ref{indep-non-gauss}:
	\[ \| \widehat{\Sigma}^{\circ}_{n} - \Sigma^\circ_n \|_2 \lesssim \sqrt{n^{-1+\alpha}\log n}. \]
	By Lemma \ref{lemma Yij} (\ref{lemma Yij-1}) and Lemma \ref{indep-X}, we have $ \max\{\norm{\widetilde U_i }_2, \norm{X_i^\circ}_2 \} \le i^{-\alpha/2}$, and
	\begin{equation*}
		\zeta_i := \widetilde U_i  - X_i^\circ = Y_{(A)}{}_0^i\delta_0 \lesssim  \exp\left( -\eta A i^{1-\alpha} \right),
	\end{equation*}
	which implies $\abs{\zeta_i} \lesssim i^{-K}$ for any $K > 0$.
	Therefore for the third term, we have
	\begin{align*}
		|  \Sigma^\circ_n - \Sigma_n  |
		& = \abs{\frac1n \sum_{i=1}^n \E\left[ \widetilde U_i ^2 + 2\widetilde U_i  \bigg(\sum_{j=1}^{i-1} X_j \bigg) \right] - \frac1n \sum_{i=1}^n \E\left[ |X_i^\circ|^2 + 2X_i^\circ \bigg(\sum_{j=1}^{i-1} X_j^\circ \bigg) \right] } \\
		& \le \abs{ \frac1n \sum_{i=1}^n \E\left[ \widetilde U_i \zeta_i + 2\widetilde U_i  \bigg(\sum_{j=1}^{i-1} \zeta_j \bigg) \right] }
		+ \abs{ \frac1n \sum_{i=1}^n \E\left[ \zeta_i X_i^\circ + 2\zeta_i \bigg(\sum_{j=1}^{i-1} X_j^\circ \bigg) \right] } \\
		& \lesssim \frac1n  \sum_{i=1}^n \norm{\widetilde U_i }_1 \sum_{j=1}^{i} \abs{\zeta_j}
		+ \frac1n \sum_{i=1}^n \abs{\zeta_i} \sum_{j=1}^{i}  \|X_j^\circ\|_1  \\
		& \lesssim 2 \cdot \bigg(\frac1n \sum_{i=1}^n i^{-\alpha/2} \bigg) \bigg( \sum_{j=1}^{n} j^{-K} \bigg)
		\lesssim  n^{-K_1},
	\end{align*}
 for any $K_1 > 0$. Similarly, for the first term,
	\begin{align*}
		\| \widehat{\Sigma}_{n, \wt U}- \widehat{\Sigma}^{\circ}_{n} \|_2
		& = \norm{\frac1n \sum_{i=1}^n \left[ \widetilde U_i ^2 + 2\widetilde U_i  \bigg(\sum_{j=t_i}^{i-1} X_j \bigg) \right] - \frac1n \sum_{i=1}^n \left[ |X_i^\circ|^2 + 2X_i^\circ \bigg(\sum_{j=t_i}^{i-1} X_j^\circ \bigg) \right] }_2 \\
		& \le \norm{ \frac1n \sum_{i=1}^n \left[ \widetilde U_i \zeta_i + 2\widetilde U_i  \bigg(\sum_{j=t_i}^{i-1} \zeta_j \bigg) \right] }_2
		+ \norm{ \frac1n \sum_{i=1}^n \left[ \zeta_i X_i^\circ + 2\zeta_i \bigg(\sum_{j=t_i}^{i-1} X_j^\circ \bigg) \right] }_2 \\
		& \lesssim \frac1n  \sum_{i=1}^n \norm{\widetilde U_i }_2 \sum_{j=t_i}^{i} \abs{\zeta_j}
		+ \frac1n \sum_{i=1}^n \abs{\zeta_i} \sum_{j=t_i}^{i}  \|X_j^\circ\|_2 
		\lesssim  n^{-K_1}.
	\end{align*}	

 For the fourth term, recall that $\epsilon^{*}_i$'s are i.i.d. with $\E[\epsilon^{*}_i] = 0$, $\E[\epsilon^{*2}_i] = S$, then
		\begin{align*}
			\Sigma_{n, \wt U}& =
			\frac1n \E\left[ \biggl(\sum_{i=1}^{n} \widetilde U_i\biggr)^2 \right]
			= \frac1n \E\left[ \bigg( S_{(A)}{}_0^n \delta_0 + \sum_{i=1}^n (1+S_{(A)}{}_i^n)\eta_i \epsilon^*_i \bigg)^2 \right] \\
			& = \frac1n \biggl( (S_{(A)}{}_0^n \delta_0)^2 + \sum_{i=1}^n (1+S_{(A)}{}_i^n)^2 \eta_i^2 S \biggr). 
		\end{align*}
		Therefore
		\begin{align*}
			|\Sigma_{n, \wt U}- \Sigma| 
			& \le \frac1n (S_{(A)}{}_0^n \delta_0)^2 + \frac1n \sum_{i=1}^n \left[ (1+S_{(A)}{}_i^n)^2 \eta_i^2 - A^{-2} \right] S \\
			& \lesssim \frac1n \cdot \mathcal{O}(1) + \frac1n \sum_{i=1}^n \left[ (1+S_{(A)}{}_i^n) \eta_i - A^{-1} \right] \cdot \mathcal{O}(1) \\
			& \lesssim \frac1n + \frac1n \sum_{i=1}^n (S_{(A)}{}_i^n \eta_i - A^{-1}).
		\end{align*}
		To bound $(S_{(A)}{}_i^n \eta_i - A^{-1})$, by Lemma \ref{lemma Yij} (\ref{lemma Yij-1}),
		\begin{align*}
			S_{(A)}{}_i^n  & =  \sum_{k = i+1}^n Y_{(A)}{}_i^k
			\asymp\sum_{k = i+1}^n \exp\left\{\frac{\eta A}{1-\alpha} \left(i^{1-\alpha} - k^{1-\alpha} \right)  \right\},
		\end{align*}
		and then
		\begin{align*}
			\sum_{k = i+1}^n \exp\left(- \frac{\eta A}{1-\alpha} k^{1-\alpha} \right)
			& \le \int_{i}^\infty \exp\left(- \frac{\eta A}{1-\alpha} u^{1-\alpha} \right) \dif u \\
			& = \frac{1}{\eta A} \biggl(\frac{1-\alpha}{\eta A}\biggr)^{\frac{\alpha}{1-\alpha}} \int_{\frac{\eta A}{1-\alpha} i^{1-\alpha} }^\infty \e^{-t} t^{\frac{\alpha}{1-\alpha}} \dif t.
		\end{align*}
		Using integration by parts, for any fixed $a \ge 1$, $\beta > 1$,
		\begin{align*}
			\int_a^\infty \e^{-x} x^\beta \dif x
			& = \e^{-a}a^\beta + \beta \int_a^\infty \e^{-x} x^{\beta - 1} \dif x \\
			& = \e^{-a}a^\beta \left( 1 + \beta a^{-1} \right) + \beta(\beta-1) \int_a^\infty \e^{-x} x^{\beta - 2} \dif x = \cdots \\
			& = \e^{-a}a^\beta \left( 1 + \beta a^{-1} + \cdots \widetilde{\beta}! a^{-\lfloor \beta \rfloor} \right) 
			+ \beta! \int_a^\infty \e^{-x} x^{\beta - \lfloor \beta \rfloor - 1} \dif x \\
			& \le \e^{-a}a^\beta \left( 1 + \beta a^{-1} + \cdots \widetilde{\beta}! a^{-\lfloor \beta \rfloor}
			+ \beta! a^{-\beta} \right),
		\end{align*}
		where $\beta ! = \beta(\beta-1)\cdots(\beta - \lfloor \beta \rfloor)$, $\widetilde{\beta}! = \beta !/(\beta - \lfloor \beta \rfloor)$.

		Therefore, for $i$ large enough,
		\begin{align*}
			{} & \int_{\frac{\eta A}{1-\alpha} i^{1-\alpha} }^\infty \e^{-t} t^{\frac{\alpha}{1-\alpha}} \dif t \\
			= {} & \exp\left( - \frac{\eta A}{1-\alpha} i^{1-\alpha}  \right) \left(\frac{\eta A}{1-\alpha} i^{1-\alpha} \right)^{\frac{\alpha}{1-\alpha}} \left( 1 +  (\eta A)^{-1} i^{-1+\alpha} + o(i^{-1+\alpha}) \right),
		\end{align*}
		which implies
		\begin{align*}
			S_{(A)}{}_i^n  & \lesssim \exp\left( \frac{\eta A}{1-\alpha} i^{1-\alpha}  \right) \sum_{k = i+1}^n \exp\left(- \frac{\eta A}{1-\alpha} k^{1-\alpha} \right) \\
			& = (\eta A)^{-1} i^\alpha \left( 1 +  (\eta A)^{-1} i^{-1+\alpha} + o(i^{-1+\alpha}) \right) 
		\end{align*}
		and hence
		\begin{align*}
			S_{(A)}{}_i^n \eta_i - A^{-1} & \lesssim A^{-1}\left( 1 +  (\eta A)^{-1} i^{-1+\alpha} + o(i^{-1+\alpha}) \right)  - A^{-1} \\
			& \lesssim i^{-1+\alpha}.
		\end{align*}
		Finally, it gives
		\begin{align*}
			|\Sigma_{n, \wt U}- \Sigma| 
			& \lesssim \frac1n + \frac1n \sum_{i=1}^n (S_{(A)}{}_i^n \eta_i - A^{-1}) \\
			& \lesssim \frac1n + \frac1n \sum_{i=1}^n i^{-1+\alpha} \asymp  n^{-1+\alpha}.
		\end{align*}
Hence, putting together,
	\begin{align*}
		\| \widehat{\Sigma}_{n, \wt U}- \Sigma \|_2
		&\le \| \widehat{\Sigma}_{n, \wt U}- \widehat{\Sigma}^{\circ}_{n} \|_2 + \| \widehat{\Sigma}^{\circ}_{n} - \Sigma^\circ_n \|_2 + \| \Sigma^\circ_n - \Sigma_{n, \wt U} \|_2+ \|\Sigma_{n, \wt U}-\Sigma\|_2 \\
		& \lesssim n^{-K_1} + \sqrt{n^{-1+\alpha} \log n} + n^{-K_1} +n^{-1+\alpha}\\
		& \lesssim \sqrt{n^{-1+\alpha} \log n}.  \tag{$\alpha > 1/2$}
	\end{align*}
\end{proof}

\subsection{With Linear Approximation: the Third Term}\label{linear}
Recall the definition of the linear approximation sequence,
\[U_i = (1 - \eta_i A) U_{i-1} + \eta_i \epsilon_i; \quad U_0 = \delta_0,\]
where $\epsilon_{i} =  \nabla F(X_{i-1}) - \nabla f(X_{i-1}, \xi_i)$. We consider the starting point $\delta_0$ to be deterministic. The corresponding estimator based on this sequence is given by
\[\widehat{\Sigma}_{n,  U}= \frac1n\sum_{i=1}^n \left[  U_i^2 + 2 U_i \biggl(\sum_{k=t_i}^{i-1}  U_k\biggr)  \right] \]
The next proposition bound the third term in \eqref{eq:decomp}.

\begin{proposition}[linear case] \label{prop linear case}
	Suppose Assumptions \ref{as1}-\ref{as3} hold, then we have
	\begin{equation}
		\| \widehat{\Sigma}_{n, U}- \widehat{\Sigma}_{n, \wt U}  \|_2 = \sqrt{ \E \abs{ \widehat{\Sigma}_{n, U}-  \widehat{\Sigma}_{n, \wt U} }^2}
		\lesssim \sqrt{n^{-1+\alpha}\log n}.
	\end{equation}
   Moreover, combining this with Proposition \ref{indep-general}, we conclude that the intermediate estimator based on the linear sequence is consistent, with
    \[\| \widehat{\Sigma}_{n, U}- \Sigma  \|_2=\sqrt{\E \abs{ \widehat{\Sigma}_{n, U}-  \Sigma }^2}
		\lesssim \sqrt{n^{-1+\alpha}\log n}.\]
\end{proposition}
\begin{proof}
    By Lemma \ref{lemma C increm} (\ref{lemma C increm-1}), Lemma \ref{lemma C iter},
	\begin{align*}
		&  \,  \| \widehat{\Sigma}_{n, U}- \widehat{\Sigma}_{n, \wt U}\|_2 \\
		= {} & \norm{\frac1n \sum_{i=1}^n \left[ U_i^2 + 2U_i \bigg(\sum_{j=t_i}^{i-1} U_j \bigg) \right] - \frac1n \sum_{i=1}^n \left[ \widetilde U_i^2 + 2\widetilde U_i \bigg(\sum_{j=t_i}^{i-1}\widetilde U_j \bigg) \right] }_2 \\
		\le {} &  \norm{ \frac1n \sum_{i=1}^n \left[ U_i\Delta_i + 2U_i \bigg(\sum_{j=t_i}^{i-1} \Delta_j \bigg) \right] }_2
		+ \norm{ \frac1n \sum_{i=1}^n \left[ \Delta_i \widetilde U_i  + 2\Delta_i \bigg(\sum_{j=t_i}^{i-1} \widetilde U_j \bigg) \right] }_2 \\
		\lesssim {} &  \frac1n \sum_{i=1}^n \left[ \norm{U_i\Delta_i}_2 + \norm{U_i \bigg(\sum_{j=t_i}^{i} \Delta_j \bigg)}_2 
		+ \| \Delta_i \widetilde U_i \|_2 + \norm{\Delta_i \bigg(\sum_{j=t_i}^{i} \widetilde U_j \bigg)}_2 \right]\\
		\le {} &  \frac1n \sum_{i=1}^n \left[ \norm{U_i}_4\norm{\Delta_i}_4 + \norm{U_i}_4 \norm{\sum_{j=t_i}^{i} \Delta_j }_4 
		+ \| \Delta_i\|_4 \|\widetilde U_i \|_4 + \norm{\Delta_i}_4 \norm{\sum_{j=t_i}^{i} \widetilde U_j }_4 \right] \\
		\lesssim {} & \frac1n \sum_{i=1}^n \left[ i^{-\alpha/2} \cdot i^{-\alpha} + i^{-\alpha/2} \sqrt{\log i}
		+ i^{-\alpha} \cdot i^{-\alpha/2} + i^{-\alpha} \sqrt{i^\alpha \log i} \right] \\
		\lesssim {} &  \frac1n \sum_{i=1}^n i^{-\alpha/2} \sqrt{\log i}
		\lesssim  \sqrt{n^{-\alpha} \log n}.
	\end{align*}
	Since $\alpha > 1/2$,  we finally have
		\[ \| \widehat{\Sigma}_{n, U}- \widehat{\Sigma}_{n, \wt U}  \|_2
		\lesssim  \sqrt{n^{-\alpha} \log n} \lesssim \sqrt{n^{-1+\alpha} \log n}.   \]

\end{proof}

\subsection{General Case: the First and Second Term}\label{general}
Recall the definition of covariance estimators with and without the sample average,
\begin{align*}
	\widehat{\Sigma}_{n} & = 
		\frac1n  \sum_{i=1}^n \Bigg[  2(\delta_i - \bar \delta_n) \bigg(\sum_{k=t_i}^{i} \delta_k - \ell_i\bar \delta_n \bigg) - (\delta_i - \bar \delta_n)^2 \Bigg] \\
	& =  \frac1n\sum_{i=1}^n \left[ 2\delta_i \biggl(\sum_{k=t_i}^{i} \delta_k\biggr) - \delta_i^2 \right]
		+ \frac1n\sum_{i=1}^n (2\ell_i - 1) \bar \delta_n^2
		-  \frac2n \sum_{i=1}^n  \biggl(\sum_{k=t_i}^{i-1} \delta_k + \ell_i \delta_i \biggr)\bar \delta_n.  \\
		\widehat{\Sigma}_{n, \delta} & = \frac1n\sum_{i=1}^n \left[  2 \delta_i \biggl(\sum_{k=t_i}^{i} \delta_k\biggr)
		- \delta_i^2  \right].
\end{align*}

We now complete the proof by bounding the first and second terms in \eqref{eq:decomp}, i.e., 
\[\|\widehat{\Sigma}_{n} - \widehat{\Sigma}_{n, \delta} \|_2,\  \text{and} \ \|  \widehat{\Sigma}_{n, \delta}-\widehat{\Sigma}_{n, U} \|_2.\]

\begin{proof}[Complete proof of the Theorem \ref{thm:general}]
For the first term, by Lemma \ref{lemma ASGD} (\ref{lemma ASGD-2}, \ref{lemma ASGD-3}), we have
\begin{align*}
	\| \widehat{\Sigma}_{n} - \widehat{\Sigma}_{n,\delta} \|_2
	& = \norm{\frac1n\sum_{i=1}^n (2\ell_i - 1) \bar \delta_n^2
	-  \frac2n \sum_{i=1}^n  \biggl(\sum_{k=t_i}^{i-1} \delta_k + \ell_i \delta_i \biggr)\bar \delta_n}_2 \\
	& \lesssim \frac1n\sum_{i=1}^n  \ell_i  \|\bar \delta_n^2\|_2
	+ \frac1n \norm{ \sum_{i=1}^n  \biggl(\sum_{k=t_i}^{i-1} \delta_k + \ell_i \delta_i \biggr)\bar \delta_n }_2 \\
	& \le \frac1n\sum_{i=1}^n  \ell_i  \|\bar \delta_n\|_4^2
	+ \frac1n \norm{ \sum_{i=1}^n  \biggl(\sum_{k=t_i}^{i-1} \delta_k + \ell_i \delta_i \biggr)}_4 \| \bar \delta_n \|_4  \\
	& \lesssim \frac1n\sum_{i=1}^n  \ell_i  \|\bar \delta_n\|_4^2
	+ \frac1n \sum_{i=1}^n \norm{ \sum_{k=t_i}^{i} \delta_k }_4 \| \bar \delta_n \|_4 + \frac1n \norm{\sum_{i=1}^n \ell_i \delta_i}_4 \| \bar \delta_n \|_4 \\
	& \le \frac1n\sum_{i=1}^n  \ell_i  \|\bar \delta_n\|_4^2
	+ \frac1n \sum_{i=1}^n \norm{ \sum_{k=t_i}^{i} \delta_k }_4 \| \bar \delta_n \|_4 + \ell_n \| \bar \delta_n \|_4^2 \\
	& \lesssim \frac1n\sum_{i=1}^n  i^\alpha \log i \cdot n^{-1}
	+ \frac1n \sum_{i=1}^n \sqrt{i^\alpha \log i} \cdot n^{-1/2}
	+ n^\alpha \log n \cdot n^{-1} \\
	& \asymp  n^{-1+\alpha}\log n + \sqrt{n^{-1+\alpha} \log n} + n^{-1+\alpha}\log n \\
	& \asymp \sqrt{n^{-1+\alpha} \log n}.
\end{align*}
For the second term, by Lemma \ref{Lemma-C-0}, \ref{lemma C iter} and  \ref{lemma diff seq}(\ref{lemma diff seq-3}), 

\begin{align*}
		&  \,  \| \widehat{\Sigma}_{n, \delta} - \widehat{\Sigma}_{n, U} \|_2 \\
		= {} & \norm{\frac1n \sum_{i=1}^n \left[ \delta_i^2 + 2\delta_i \bigg(\sum_{j=t_i}^{i-1} \delta_j \bigg) \right] - \frac1n \sum_{i=1}^n \left[ U_i^2 + 2U_i \bigg(\sum_{j=t_i}^{i-1} U_j \bigg) \right] }_2 \\
		\le {} &  \norm{ \frac1n \sum_{i=1}^n \left[ \delta_i s_i + 2\delta_i \bigg(\sum_{j=t_i}^{i-1} s_j \bigg) \right] }_2
		+ \norm{ \frac1n \sum_{i=1}^n \left[ s_i U_i  + 2 s_i \bigg(\sum_{j=t_i}^{i-1} U_j \bigg) \right] }_2 \\
		\lesssim {} &  \frac1n \sum_{i=1}^n \left[ \norm{\delta_i s_i}_2 + \norm{\delta_i \bigg(\sum_{j=t_i}^{i-1} s_j \bigg)}_2 
		+ \| s_i U_i \|_2 + \norm{s_i \bigg(\sum_{j=t_i}^{i-1} U_j \bigg)}_2 \right]\\
		\le {} &  \frac1n \sum_{i=1}^n \left[ \norm{\delta_i}_4\norm{s_i}_4 + \norm{\delta_i}_4 \norm{\sum_{j=t_i}^{i-1} s_j }_4 
		+ \| s_i\|_4 \| U_i \|_4 + \norm{s_i}_4 \norm{\sum_{j=t_i}^{i-1} U_j }_4 \right] \\
		\lesssim {} & \frac1n \sum_{i=1}^n \left[ \norm{\delta_i}_4 \sum_{j=t_i}^{i} \norm{ s_j }_4 
		+ \norm{s_i}_4 \sum_{j=t_i}^{i} \norm{U_j }_4 \right]
		\\
		\lesssim {} & \frac1n \sum_{i=1}^n \left[ i^{-\alpha/2} \ell_i i^{-\alpha}
		+ i^{-\alpha} \ell_i i^{-\alpha/2} \right] \\
		\lesssim {} &  \frac1n \sum_{i=1}^n i^{-\alpha/2} \log i
		\lesssim n^{-\alpha/2} \log n. 
	\end{align*}

Since $\alpha > 1/2$, 
	\[ n^{-\alpha/2} \log n \lesssim \sqrt{n^{-1+\alpha} \log n}.\]

Now combined with bound on the third and the fourth term, as shown in Propositions ~\ref{prop linear case} and \ref{indep-general}, we finally have 
\[
	\| \widehat{\Sigma}_{n} - \Sigma \|_2 
 \lesssim \sqrt{n^{-1+\alpha} \log n}.
\]
\end{proof}

\subsection{Technical Lemmas}

\begin{lemma}[Lipschitz continuity of $\nabla F$]\label{determinelip}
Under Assumptions \ref{as1}-\ref{as3}, we have
 $$ \|\nabla F(X_1)-\nabla F(X_2)\|\leq L |X_1-X_2|.$$
\end{lemma}
\begin{proof}

By the stochastic Lipschitz continuity, 
\[\E  \|\nabla f(X_1, \xi) - \nabla f(X_2, \xi) \|^2_2 \leq L^2 |X_1-X_2|^2, \quad \text{for all }\ X_1,X_2\in \mathbb{R}^d.\]
By the convexity of \(| \cdot |^2\) and Jensen's inequality,
\begin{align*}
\|\nabla F(X_1)-\nabla F(X_2)\|^2 &= | \E_{\xi} ( \nabla f(X_1, \xi) - \nabla f(X_2, \xi)  ) |^2  \\
&\leq \E_{\xi} | \nabla f(X_1, \xi) - \nabla f(X_2, \xi) |^2 \\
&\leq L^2 | X_1-X_2|^2.
\end{align*}

\end{proof}
  
\begin{lemma}[moment bounds for the error sequence] \label{Lemma-C-0} Under Assumptions \ref{as1}-\ref{as3}, for $1\leq q \leq 8$, the error sequence $\delta_n = X_n - x^*$ satisfies $ \| \delta_n \|_q\leq n^{-\alpha/2}$.
\end{lemma}

\begin{proof}
The error sequence can be written as,
		\begin{align*}
  \delta_i  &= \delta_{i-1} - \eta_i \nabla f(X_{i-1}, \xi_i) \cr
	  &= \delta_{i - 1} - \eta_{i} \nabla F(X_{i - 1}) + \eta_{i} \epsilon_i.
		\end{align*}
By Rio's inequality \citep{rio_moment_2009}, since $\E [\epsilon_i | \delta_{i - 1} - \eta_{i} \nabla F(X_{i - 1}) ] =0$, we have
$$ \|\delta_{i}\|_{q}^{2} \leq \|\delta_{i - 1} - \eta_{i} \nabla F (X_{i - 1})\|_{q}^{2} + (q - 1) \eta_{i}^{2} \|\epsilon_i\|_{q}^{2}.$$
		Further by strong convexity and stochastic Lipschitz continuity, we have 
		\begin{align*}
			\|\delta_{i}\|_{q}^{2} &\leq \|\delta_{i - 1} - \eta_{i} \nabla F (X_{i - 1})\|_{q}^{2} + (q - 1) \eta_{i}^{2} \|\epsilon_i\|_{q}^{2}\cr
			&\leq (1 - \eta_{i} c_{1}) \|\delta_{i - 1}\|_{q}^{2} + 2 (q - 1) \eta_{i}^{2} (\|\epsilon_i^*\|_{q}^{2} + L_{\epsilon}^{2} \|\delta_{i - 1}\|_{q}^{2})\cr
			&\leq (1 - \eta_{i} c_{2}) \|\delta_{i - 1}\|_{q}^{2} + 2 (q - 1) \eta_{i}^{2} \|\epsilon_i^*\|_{q}^{2}, 
		\end{align*}
		where $c_1$ and $c_2$ are positive constants depending only on $L_{\epsilon}, q$ and $\eta$. Finally, by Lemma \ref{lemma Yij},
		\begin{align*}
			\|\delta_{i}\|_{q}^{2} &\leq \prod_{k = 1}^{i} (1 - \eta_{k} c_{2}) \delta_{0}^{2} + 2 (q - 1) \|\epsilon_i^*\|_{q}^{2} \sum_{j = 1}^{i} \eta_{j}^{2} \prod_{k = j + 1}^{i} (1 - \eta_{k} c_{2}) \cr
         &\asymp {Y_{(c_2)}}_0^i +  \sum_{j = 1}^{i} {Y_{(c_2)}}_j^i j^{-2\alpha} \asymp i^{-\alpha}.
		\end{align*}
\end{proof}

\begin{lemma}[moment bounds for increments] \label{lemma C increm}
Recall that
\begin{itemize}

  		\item $\epsilon_k^* = - \nabla f(x^*, \xi_k)$ is \textit{i.i.d}.
    	\item $v_k = \epsilon_k - \epsilon_k^*$ is a martingale differences sequence.

\end{itemize}
	\vphantom{1}
	\begin{enumerate}
		\item For all $i \in \N^+$ and $m = 1, 2, 4$, we have \label{lemma C increm-1}
		\begin{equation}
			\norm{\epsilon_i}_m = \mathcal{O}(1),
			\quad
			\norm{\epsilon_i^*}_m = \mathcal{O}(1),
			\quad
			\norm{v_i}_m \lesssim i^{-\alpha/2}. 
		\end{equation}
		\item For all $i \not= j$, we have $\E[\epsilon_i\epsilon_j^*] = \E[\epsilon_i v_j] = \E[\epsilon_i^* v_j] = 0$.
		\label{lemma C increm-2}
	\end{enumerate}
\end{lemma}
\begin{proof}
	\vphantom{1}
	\begin{enumerate}
		\item Due to Lyapunov inequality, it suffices to consider $m = 4$. 
	
		The bound for $\epsilon_i$ is obtained by stochastic Lipschitz continuity: 
        $$(\E_{n-1}|\epsilon_n|^4)^{1/4} \le (\E_{n-1}|v_n|^4)^{1/4}+(\E_{n-1}|\epsilon^*_n|^4)^{1/4}\lesssim \delta_{n-1}+1,$$
        and $\| \delta_{n-1} \|_4 \lesssim (n-1)^{-\alpha/2}.$
		The bound for $v_i$ is obtained by Lipschitz continuity,
		\begin{align*}
			\norm{v_i}_4
			& = \norm{\epsilon_i - \epsilon_i^*}_4
			= \norm{\nabla F(X_{i-1}) - \nabla F(x^*) - ( \nabla f(X_{i-1}, \xi_i) - \nabla f(x^*, \xi_i) ) }_4 \\
			&\le \norm{\nabla F(X_{i-1}) - \nabla F(x^*)}_4 + \norm{\nabla f(X_{i-1}, \xi_i) - \nabla f(x^*, \xi_i)}_4 \\
			& \lesssim \norm{X_{i-1} - x^*}_4 = \norm{\delta_{i-1}}_4 \lesssim  i^{-\alpha/2}.
		\end{align*}
		Then a bound for $\epsilon^*_i$ is derived by $\norm{\epsilon_i^*}_4  =  \norm{\epsilon_i - v_i}_4 \le \norm{\epsilon_i}_4 + \norm{v_i}_4 \lesssim \mathcal{O}(1)$.

		\item Recall that $\xi_i$'s are i.i.d., $\epsilon_i$'s are martingale difference, and $\epsilon_i^*$'s are i.i.d., then we have
		\[ \E[\epsilon_i\epsilon_j^*] = \begin{dcases*}
			\E\left[ \E_{i-1}[\epsilon_i\epsilon_j^*] \right]
			= \E\left[ \E_{i-1}[\epsilon_i] \epsilon_j^* \right] = 0 & if $i > j$, \\
			\E\left[ \E_{i}[\epsilon_i\epsilon_j^*] \right]
			= \E\left[ \epsilon_i \E_{i}[\epsilon_j^*] \right] = 0 & if $i < j$.
		\end{dcases*} \]
		Hence $\E[\epsilon_i v_j] = \E[\epsilon_i (\epsilon_j - \epsilon_j^*)] = 0$, $\E[\epsilon_i^* v_j] = \E[\epsilon_i^* (\epsilon_j - \epsilon_j^*)] = 0$
	\end{enumerate}

\end{proof}

\begin{itemize}
	\item Recall that
	\begin{align*}
		U_n & = (1 - \eta_n A)U_{n-1} + \eta_n \epsilon_n \\
		\widetilde U_n & = (1 - \eta_n A) \widetilde U_{n-1} + \eta_n \epsilon_n^* \\
		\Delta_n & = (1 - \eta_n A)\Delta_{n-1} + \eta_n v_n \\
	\end{align*}
\end{itemize}

\begin{lemma}[moment bounds for iterates] \label{lemma C iter}
	\vphantom{1}
	\begin{enumerate}
		\item For all $n \in \N^+$ and $m = 1, 2, 4$, we have \label{lemma C iter-1}
		\begin{equation}
			\norm{U_n}_m \lesssim n^{-\alpha/2},
			\quad
			\|\widetilde U_n\|_m \lesssim n^{-\alpha/2},
			\quad
			\norm{\Delta_n}_m \lesssim n^{-\alpha}. 
		\end{equation}

		\item For all $i \in \N^+$ and $m = 1, 2, 4$, we have \label{lemma C iter-2}
		\begin{align}
			\norm{\sum_{j=t_i}^i U_j}_m & \lesssim t_i^{\alpha/2} + \sqrt{\ell_i} \asymp \sqrt{i^{\alpha}\log i} , \\
			\norm{\sum_{j=t_i}^i \widetilde U_j}_m & \lesssim \sqrt{\ell_i} \asymp \sqrt{i^{\alpha}\log i}, \\
			\norm{\sum_{j=t_i}^i \Delta_j}_m & \lesssim \mathcal{O}(1) + \sqrt{i^{-\alpha}\ell_i}  =  \sqrt{\log i}. 
		\end{align}
	\end{enumerate}
\end{lemma}
\begin{proof}
	Due to Lyapunov inequality, it suffices to consider $m = 4$.
	\begin{enumerate}
		\item For simplicity, write $Y_i^j := Y_{(A)}{}_i^j$. Notice that $U_n$, $\widetilde U_n$, $\Delta_n$ can be written as
		\[ U_n = Y_0^n \delta_0 + \sum_{i=1}^n Y_i^n \eta_i \epsilon_i,
		\quad
		\widetilde U_n = Y_0^n \delta_0 + \sum_{i=1}^n Y_i^n \eta_i \epsilon_i^*,
		\quad
		\Delta_n = \sum_{i=1}^n Y_i^n \eta_i v_i.  \]
		Since $\epsilon_i$'s are martingale difference, by Lemma \ref{lemma Yij} (\ref{lemma Yij-1}, \ref{lemma Yij-2}), Lemma \ref{lemma Burkholder} (Burkholder inequality), Lemma \ref{lemma C increm} (\ref{lemma C increm-1}), we have
		\begin{align*}
			\norm{U_n}_4 & \le Y_0^n \delta_0 + \norm{\sum_{i=1}^n Y_i^n \eta_i \epsilon_i}_4  \\
			& \lesssim Y_0^n \delta_0 +  \sqrt{\sum_{i=1}^n \norm{Y_i^n \eta_i \epsilon_i}_4^2 }  \\
			& \lesssim  Y_0^n + \sqrt{\sum_{i=1}^n \abs{Y_i^n}^2 \abs{i^{-\alpha}}^2  }  \\
			& \lesssim  \exp\left(-\eta A i^{1-\alpha} \right) + \sqrt{n^{-\alpha}} \\
			& \asymp n^{-\alpha/2}.
		\end{align*}
		$\widetilde U_n$ has the same bound as $U_n$ since both $\norm{\epsilon_i}_4, \norm{\epsilon_i^*}_4 = \mathcal{O}(1)$. For $\Delta_n$, similarly,
		\begin{align*}
			\norm{\Delta_n}_4 & = \norm{\sum_{i=1}^n Y_i^n \eta_i v_i}_4 \\
			& \lesssim \sqrt{\sum_{i=1}^n \norm{Y_i^n \eta_i v_i}_4^2 }  \\
			& \lesssim  \sqrt{\sum_{i=1}^n \abs{Y_i^n}^2 \abs{i^{-\alpha}}^2 \cdot i^{-\alpha}  }  \\
			& \lesssim  \sqrt{n^{-2\alpha}}  \asymp  n^{-\alpha}.    
		\end{align*}

		\item For simplicity, write $S_i^j := S_{(A)}{}_i^j$. Notice that the following sums can be written as
		\begin{align*}
			\sum_{j = t_i}^i U_j & = \sum_{j = t_i}^i \biggl( Y_{t_i -1}^j U_{t_i - 1} + \sum_{p = t_i}^j Y_p^j \eta_p \epsilon_p \biggr) = S_{t_i -1}^i U_{t_i - 1} + \sum_{p = t_i}^i (1 + S_p^i) \eta_p \epsilon_p, \\
			\sum_{j = t_i}^i \widetilde U_j & = \sum_{j = t_i}^i \biggl( Y_{t_i -1}^j \widetilde U_{t_i - 1} + \sum_{p = t_i}^j Y_p^j \eta_p \epsilon_p^* \biggr) = S_{t_i -1}^i \widetilde U_{t_i - 1} + \sum_{p = t_i}^i (1 + S_p^i) \eta_p \epsilon_p^*, \\
			\sum_{j = t_i}^i \Delta_j & = \sum_{j = t_i}^i \biggl( Y_{t_i -1}^j \Delta_{t_i - 1} + \sum_{p = t_i}^j Y_p^j \eta_p v_p \biggr) = S_{t_i -1}^i \Delta_{t_i - 1} + \sum_{p = t_i}^i (1 + S_p^i) \eta_p v_p.
		\end{align*}
		Again, by Lemma \ref{lemma Yij} (\ref{lemma Yij-3}), Lemma \ref{lemma Burkholder} (Burkholder inequality), Lemma \ref{lemma C increm} (\ref{lemma C increm-1}), we have
		\begin{align*}
			\norm{\sum_{j=t_i}^i U_j}_4 & \le \abs{S_{t_i -1}^i} \norm{U_{t_i - 1}}_4 
			+ \norm{\sum_{p = t_i}^i (1 + S_p^i) \eta_p \epsilon_p}_4 \\
			& \lesssim \abs{S_{t_i -1}^i} \norm{U_{t_i - 1}}_4 
			+ \sqrt{\sum_{p = t_i}^i \norm{(1 + S_p^i) \eta_p \epsilon_p}_4^2  } \\
			& = \abs{S_{t_i -1}^i} \norm{U_{t_i - 1}}_4 
			+ \sqrt{\sum_{p = t_i}^i (1 + S_p^i)^2 \eta_p^2 \norm{\epsilon_p}_4^2  }  \\
			& \lesssim t_i^\alpha  (t_i - 1)^{-\alpha/2} + \sqrt{\sum_{p = t_i}^i (1 + (p+1)^\alpha)^2 p^{-2\alpha} \cdot \mathcal{O}(1)  } \\
			& \asymp t_i^{\alpha/2} + \sqrt{\ell_i}.
		\end{align*}
		Notice that $t_i^{\alpha/2} \le i^{\alpha/2}$ and $\ell_i \asymp i^\alpha \log i$, so the above is bounded by $\sqrt{\ell_i} \asymp \sqrt{i^{\alpha}\log i}$.

		$\|\sum_{j=t_i}^i \widetilde U_j\|_4$ has the same bound as $\|\sum_{j=t_i}^i U_j\|_4$ since $\| \widetilde U_{t_i-1} \|_4, \| U_{t_i-1} \|_4 \lesssim (t_i - 1)^{-\alpha/2}$ and $\|\epsilon_p^*\|_4, \|\epsilon_p\|_4 = \mathcal{O}(1)$. 
        
        
		For $\Delta_i$, similarly, also by Lemma \ref{lemma batch sum}, we have 
		\begin{align*}
			\norm{\sum_{j=t_i}^i \Delta_j}_4 & \le \abs{S_{t_i -1}^i} \norm{\Delta_{t_i - 1}}_4 
			+ \norm{\sum_{p = t_i}^i (1 + S_p^i) \eta_p v_p}_4 \\
			& \lesssim \abs{S_{t_i -1}^i} \norm{\Delta_{t_i - 1}}_4 
			+ \sqrt{\sum_{p = t_i}^i \norm{(1 + S_p^i) \eta_p v_p}_4^2  } \\
			& = \abs{S_{t_i -1}^i} \norm{\Delta_{t_i - 1}}_4 
			+ \sqrt{\sum_{p = t_i}^i (1 + S_p^i)^2 \eta_p^2 \norm{v_p}_4^2  }  \\
			& \lesssim t_i^\alpha  (t_i - 1)^{-\alpha} + \sqrt{\sum_{p = t_i}^i (1 + (p+1)^\alpha)^2 p^{-2\alpha} \cdot p^{-\alpha}  } \\
			& \asymp \mathcal{O}(1) + \sqrt{i^{-\alpha}\ell_i} \asymp  \sqrt{\log i}.
		\end{align*}

		\begin{notes}
			For (ii), applying triangular inequality directly will give a loose bound (but still tight enough to use). For example, by Lemma \ref{lemma batch sum},
			\begin{align*}
				\norm{\sum_{j=t_i}^i \Delta_j}_4 & \le \sum_{j=t_i}^i  \norm{\Delta_j}_4
				\lesssim \sum_{j=t_i}^i j^{-\alpha} \lesssim i^{-\alpha}\ell_i \asymp \log i.
			\end{align*} 
		\end{notes}
	\end{enumerate}
	
\end{proof}

Recall that $\delta_i = U_i + s_i$:
\begin{itemize}
		\item $\delta_i  = \delta_{i-1} - \eta_i \nabla f(X_{i-1}, \xi_i)$
		\item $U_i = (1 - \eta_i A) U_{i-1} + \eta_i \epsilon_i,  \quad  U_0 = \delta_0$
		\item $s_i = (1 - \eta_i A) s_{i-1} - \eta_i r_i, \quad  s_0 = 0$
	\end{itemize}
	where
	\begin{itemize}
		\item $\epsilon_i = \nabla F(X_{i-1}) - \nabla f(X_{i-1}, \xi_i)$ \quad 
		\item $r_i = \nabla F(X_{i-1}) - A\delta_{i-1}$
	\end{itemize}

\begin{lemma}[difference sequence] \label{lemma diff seq}
	\vphantom{1}
	\begin{enumerate}
		\item For all $i \in \N^+$ and $m = 1, 2, 4$, we have \label{lemma diff seq-1}
		\begin{equation}
			\norm{r_i}_m \lesssim i^{-\alpha}. 
		\end{equation}
		\item For all $n \in \N^+$ and $m = 1, 2, 4$, we have \label{lemma diff seq-3}
		\begin{equation}
			\norm{s_n}_m \lesssim n^{-\alpha}. 
		\end{equation}
	\end{enumerate}
\end{lemma}
\begin{proof}
	\vphantom{1}
	\begin{enumerate}
		\item Recall that by Taylor's expansion 
		\begin{align*}
			\nabla F(X_{i-1}) & = \nabla F(x^*) + \nabla^2 F(x^*) (X_{i-1} - x^*) + O( \abs{X_{i-1} - x^*}^2 ) \\
			& =  A \delta_{i-1} + O(\delta_{i-1}^2),
		\end{align*}
		therefore
		 \[  \norm{r_i}_4 = \norm{\nabla F(X_{i-1}) - A \delta_{i-1}}_4
		\lesssim \|\delta_{i-1}^2\|_4
		= \|\delta_{i-1}\|_8^2 \lesssim i^{-\alpha}. \] 
		
		\item For simplicity, write $Y_i^j := Y_{(A)}{}_i^j$. Notice that $s_n$ can be written as
		\[ s_n = -\sum_{i=1}^n Y_i^n \eta_i r_i,  \]
		therefore
		\[ \norm{s_n}_4 \le \sum_{i=1}^n |Y_i^n| \eta_i \norm{r_i}_4
		\lesssim \sum_{i=1}^n |Y_i^n| i^{-2\alpha} \lesssim n^{-\alpha}. \]
	\end{enumerate}
	
\end{proof}

Recall that
	\begin{align*}
		\widehat{\Sigma}_{n, U} & = \frac1n\sum_{i=1}^n \left[ U_i^2 + 2 U_i \biggl(\sum_{k=t_i}^{i-1} U_k\biggr)  \right]. \\
		\widehat{\Sigma}_{n, \delta} & = \frac1n\sum_{i=1}^n \left[ \delta_i^2 + 2 \delta_i \biggl(\sum_{k=t_i}^{i-1} \delta_k\biggr)  \right]. \\
	\end{align*}

\begin{lemma}[moments for ASGD] For all $n \in \N^+$, we have \label{lemma ASGD}
	\begin{enumerate}
		\item $\E[\bar \delta_n] \lesssim  n^{-\alpha}$. \label{lemma ASGD-1}
		\item $\| \bar \delta_n \|_4 \lesssim n^{-1/2}$. \label{lemma ASGD-2}
		\item For all $i \in \N^+$, \label{lemma ASGD-3}
		\begin{equation}
			\norm{\sum_{k=t_i}^{i} \delta_k}_4 \lesssim t_i^{\alpha/2} + \sqrt{ \ell_i}  \asymp  \sqrt{i^\alpha \log i}.
		\end{equation}
	\end{enumerate}
\end{lemma}
\begin{proof}
	\vphantom{1}
	\begin{enumerate}
		\item By Lemma \ref{lemma diff seq} (\ref{lemma diff seq-3}), we have $\E[s_n] \le \norm{s_n}_1 \lesssim n^{-\alpha}$. Also recall that
		\[ \E[U_n] = \E\left[Y_{(A)}{}_0^n \delta_0 + \sum_{i=1}^n Y_{(A)}{}_i^n \eta_i \epsilon_i \right] = Y_{(A)}{}_0^n \delta_0  \lesssim \exp\left( -\eta A n^{1-\alpha} \right).  \]
		Therefore $\E[\delta_n] = \E[U_n] + \E[s_n] \lesssim n^{-\alpha}$, which implies
		\[ \E[\bar\delta_n] = \frac1n \sum_{i=1}^n \E[\delta_i] \lesssim n^{-\alpha}. \]

		\item Recall that
		\[ \bar U_n = \frac1n \bigg( S_{(A)}{}_0^n \delta_0 + \sum_{i=1}^n (1+S_{(A)}{}_i^n)\eta_i \epsilon_i \bigg), \]
		therefore by Lemma \ref{lemma Yij} (\ref{lemma Yij-3}), Lemma \ref{lemma Burkholder} (Burkholder's inequality), Lemma \ref{lemma C increm},
		\begin{align*}
			\|\bar U_n\|_4 & \le \frac1n \abs{S_{(A)}{}_0^n \delta_0}
			+ \frac1n \norm{\sum_{i=1}^n (1+S_{(A)}{}_i^n)\eta_i \epsilon_i}_4 \\
			& \lesssim \frac1n \abs{S_{(A)}{}_0^n} + \frac1n \sqrt{\sum_{i=1}^n \norm{(1+S_{(A)}{}_i^n)\eta_i \epsilon_i}_4^2} \\
			& \le \frac1n \abs{S_{(A)}{}_0^n} + \frac1n \sqrt{\sum_{i=1}^n (1+S_{(A)}{}_i^n)^2 \eta_i^2 \norm{\epsilon_i}_4^2 } \\
			& \lesssim \frac1n \cdot \mathcal{O}(1) + \frac1n \sqrt{\sum_{i=1}^n (i + 1)^{2\alpha} \cdot i^{-2\alpha} \cdot \mathcal{O}(1) } \\
			&\lesssim n^{-1} + n^{-1/2} \asymp n^{-1/2}.
		\end{align*}
		Also by Lemma \ref{lemma diff seq} (\ref{lemma diff seq-3}), we have
		\[ \norm{\bar s_n}_4 \le \frac1n \sum_{i=1}^n \norm{s_i}_4 \lesssim \frac1n \sum_{i=1}^n i^{-\alpha}
		\asymp n^{-\alpha}. \]
		Therefore $\|\bar \delta_n\|_4 \le \|\bar U_n\|_4 + \norm{\bar s_n}_4
		\lesssim n^{-1/2} + n^{-\alpha} \asymp n^{-1/2}$.

		\item By Lemma \ref{lemma batch sum}, Lemma \ref{lemma diff seq} (\ref{lemma diff seq-3}), we have
		\[ \norm{\sum_{k=t_i}^{i} s_k}_4 \le \sum_{k=t_i}^{i} \norm{s_k}_4
		\lesssim \sum_{k=t_i}^{i} k^{-\alpha} \lesssim i^{-\alpha} \ell_i \asymp \log i. \]
		In combination with Lemma \ref{lemma C iter} (\ref{lemma C iter-2}), we get
		\[ \norm{\sum_{k=t_i}^{i} \delta_k}_4
		\le \norm{\sum_{k=t_i}^{i} U_k}_4 + \norm{\sum_{k=t_i}^{i} s_k}_4 
		\lesssim  t_i^{\alpha/2} + \sqrt{\ell_i} + i^{-\alpha} \ell_i \asymp \sqrt{i^\alpha \log i}.  \]
	\end{enumerate}
\end{proof}

\section{PROOF OF AUXILIARY RESULTS}\label{auxiliary}

\begin{proof}[Proof of Lemma \ref{lemma Yij}]
\vphantom{1}
\begin{enumerate}
	\item Clearly it holds for $j = i$. If $j < i$, noticing that $-x \ge \log(1-x)\ge -x-x^2/2$ for all $x \in [0,1/2]$, and the set $\{k: \lambda\eta_k > 1/2\}$ is finite, then we have 
 
\begin{align*}
|Y_{(\lambda)}{}_j^i| & = \prod_{k = j+1}^i |1- \lambda\eta_k|
		\lesssim \exp\bigg( - \lambda \sum_{k=j+1}^i \eta_k \bigg), 
\end{align*}
and
\begin{align*}
\log|Y_{(\lambda)}{}_j^i| & = \sum_{k = j+1}^i \log|1- \lambda\eta_k|
		\ge - \lambda \sum_{k=j+1}^i \eta_k - \lambda^2 \sum_{k=j+1}^i \eta_k^2+ \mathcal{O}(1)
\end{align*}
which implies 
 \begin{align*}
|Y_{(\lambda)}{}_j^i| & = \exp\bigg(\sum_{k = j+1}^i \log|1- \lambda\eta_k|\bigg)
		\gtrsim \exp\bigg( - \lambda \sum_{k=j+1}^i \eta_k - \lambda^2 \sum_{k=j+1}^i \eta_k^2\bigg) \gtrsim \exp\bigg( - \lambda \sum_{k=j+1}^i \eta_k \bigg)
\end{align*}
since $\sum_{k=j+1}^i \eta_k^2$ is uniformly bounded for all $i$ and $j$. As a result,
	\begin{align*}
		|Y_{(\lambda)}{}_j^i| & = \prod_{k = j+1}^i |1- \lambda\eta_k|
		\asymp \exp\bigg( - \lambda \sum_{k=j+1}^i \eta_k \bigg) \\
		 & = \exp\bigg( - \lambda\eta\sum_{k=j+1}^i k^{-\alpha} \bigg)  
		 \asymp \exp\left( - \lambda\eta \int_{j+1}^{i} u^{-\alpha} \dif u \right) \\
		 & = \exp\left\{ - \frac{\lambda\eta}{1 - \alpha} \left(i^{1-\alpha} - (j+1)^{1-\alpha} \right)  \right\} \\
		 & = \exp\left\{ \frac{\lambda\eta}{1 - \alpha} \left(j^{1-\alpha} - i^{1-\alpha}\right) + \mathcal{O}(1) \right\}
		  \\
		& \asymp \exp\left\{ \frac{\lambda\eta}{1 - \alpha} \left(j^{1-\alpha} - i^{1-\alpha}\right) \right\}.
	\end{align*}
	Further, by Taylor series,
	\begin{align*}
		& \exp\left\{ \frac{\lambda\eta}{1 - \alpha} (j^{1-\alpha} - i^{1-\alpha}) \right\} \nonumber \\
		= {} & \exp\left\{ -\lambda\eta i^{-\alpha}(i-j) - \frac12 \lambda\eta\alpha u^{-\alpha - 1}(i-j)^2 \right\}
		\tag{for some $u \in (j, i)$} \nonumber \\
		\le {} & \exp\left\{ -\lambda\eta i^{-\alpha}(i-j) \right\}.
	\end{align*}
 The inequality becomes $\asymp$ if $Ci^{\alpha}\log i\ge i-j$ since for sufficiently large $i$, $Ci^{\alpha}\log i \le i/2$, and in this case 
    $$u^{-\alpha-1}(i-j)^2\leq C^2i^{2\alpha}(\log i)^2 (i-Ci^{\alpha}\log i)^{-\alpha-1}\lesssim 2^{\alpha+1}C^2i^{\alpha-1}(\log i)^2 \rightarrow 0.$$
	\item Since $\exp\left(\beta j^{1-\alpha}\right) j^{-\gamma\alpha}$ is ultimately increasing in $j$, we have
	\begin{align}
		\sum_{j=1}^i \exp\left(\beta j^{1-\alpha}\right) j^{-\gamma\alpha}
		& \asymp \int_{1}^{i+1} \exp\left(\beta t^{1-\alpha}\right) t^{-\gamma\alpha} \dif t
		\nonumber \\
		& \asymp \int_{\beta}^{\beta (i+1)^{1-\alpha}} \e^u u^{-\frac{(\gamma - 1)\alpha}{1-\alpha}} \dif u.  \label{lemma Yij-int}
	\end{align}
	Using integration by parts, we have the following:
	\begin{align*}
		&\int_{\beta}^{\beta(i+1)^{1-\alpha}} \e^u u^{-\frac{(\gamma - 1)\alpha}{1-\alpha}} \dif u \\
		& \asymp  \e^u u^{-\frac{(\gamma - 1)\alpha}{1-\alpha}} \Big|_{\beta}^{\beta(i+1)^{1-\alpha}}
		+ \int_{\beta}^{\beta(i+1)^{1-\alpha}} \e^u u^{-\frac{(\gamma - 1)\alpha}{1-\alpha} - 1} \dif u \\
		& \asymp \exp\left\{ \beta(i+1)^{1-\alpha} \right\} (i+1)^{-(\gamma - 1)\alpha}+ (\beta(i+1)^{1-\alpha}-\beta)\exp\left\{ \beta(i+1)^{1-\alpha} \right\} (i+1)^{-(\gamma - 1)\alpha-(1-\alpha)}\\
		& \asymp \exp\left( \beta i^{1-\alpha} \right) i^{-(\gamma - 1)\alpha}.
	\end{align*}
	Therefore by (\ref{lemma Yij-1}), we have
	\begin{align*}
		\sum_{j=1}^i |Y_{(\lambda)}{}_j^i|^\beta |j^{-\alpha}|^\gamma 
		& \asymp \sum_{j=1}^i \exp\left\{\frac{\lambda\eta \beta}{1-\alpha} \left(j^{1-\alpha} - i^{1-\alpha} \right)  \right\} j^{-\gamma\alpha} \\
		& = \exp\left( - \frac{\lambda\eta \beta}{1-\alpha} i^{1-\alpha} \right)  \sum_{j=1}^i \exp\left(\frac{\lambda\eta \beta}{1-\alpha} j^{1-\alpha}  \right) j^{-\gamma\alpha} \\
		& \asymp  i^{-(\gamma - 1)\alpha}.
	\end{align*}
	
	\item (See Lemma A.2. in \citep{zhu2023online})
\end{enumerate}

\end{proof}

\begin{proof}[Proof of Lemma \ref{lemma batch sum}]
	Recall that $\ell_i = i - t_i + 1 \asymp i^\alpha \log i$, then by Taylor series,
	\begin{align*}
		\sum_{j=t_i}^i j^{-\gamma} 
		& \asymp i^{1-\gamma} - t_i^{1-\gamma} \\
		& = (1-\gamma) i^{-\gamma} (\ell_i - 1) + \frac12\gamma(1-\gamma)u^{-\gamma-1} (\ell_i - 1)^2  \tag{$u \in (t_i, i)$} \\
		& \lesssim i^{-\gamma} \ell_i + (i - i^\alpha \log i)^{-\gamma - 1} \ell_i^2 \\
		& \asymp  i^{-\gamma} \ell_i + i^{-\gamma-1} 
		\ell_i \cdot i^\alpha \log i \\
		& \asymp i^{-\gamma}\ell_i \asymp i^{\alpha-\gamma}\log i. 
		\tag{since $\alpha < 1$}
	\end{align*}

\end{proof}

\section{ADDITIONAL EXPERIMENT RESULTS }\label{secexp}

\begin{figure}[ht]
		\centering \includegraphics[width=0.3\textwidth]{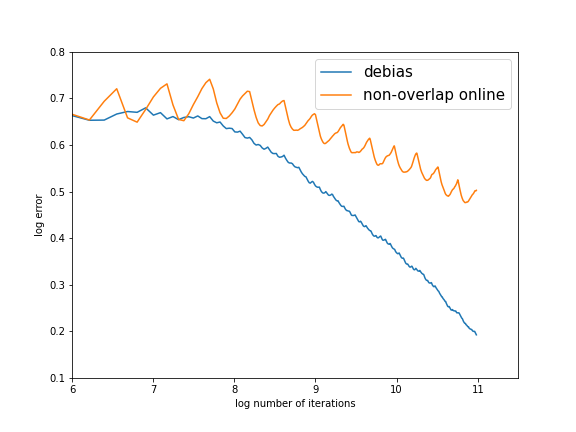} \includegraphics[width=0.3\textwidth]{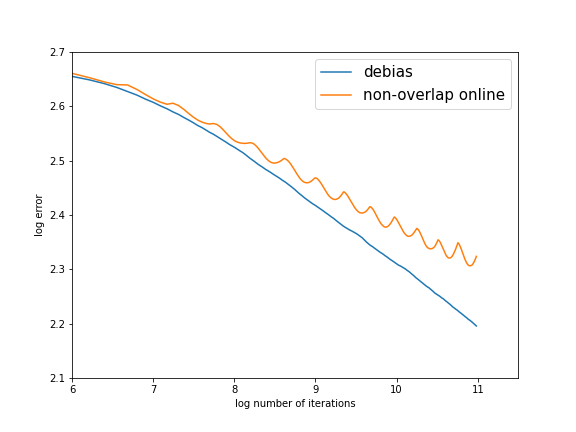}
\includegraphics[width=0.3\textwidth]{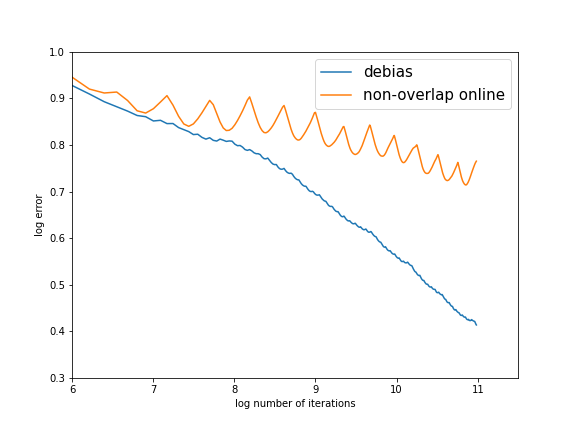} 
		\caption{Log-log Plots for Estimation Error in Frobenius Norm ($d=5$). Left: Linear Regression. Middle: Logistic Regression with. Right: Expectile Regression.  }
		\label{fig:MSE} 
	\end{figure}

We conducted the experiments in Python version 3.8.8 (2021-02-19) on a MacBook Air with a GPU Apple M1, 4 performance and 4 efficiency cores, and 8 GB LPDDR4 memory, equipped with macOS Big Sur version 11.5.1.

We present additional Log-log plots across different models and dimensional settings in Figure \ref{fig:MSE}. They clearly indicate that the de-biased estimator outperforms the non-overlap online method, achieving both sharper convergence rates (reflected in the slope) and lower estimation errors.

\begin{table*}[h]
\centering
\caption{Comparison of Empirical Coverage (Nominal Level $95\%$) across Different Models and Dimensions ($d=5, 20, 50$).}
\label{tb:empirical_coverage}
\begin{tabular}{l|l|c|c|c|}
  &   & $d=5$ & $d=20$ & $d=50$ \\
\hline
\textbf{Linear}
    & De-biased    & 0.9236 & 0.9321 & 0.9390 \\
    & Online BM    & 0.8796 & 0.8946 & 0.9127 \\
\hline
\textbf{Logistic}
    & De-biased    & 0.8872 & 0.8514 & 0.8517 \\
    & Online BM    & 0.8536 & 0.8178 & 0.8305 \\
\hline
\textbf{Expectile}
    & De-biased    & 0.9033 & 0.9127 & 0.9203 \\
    & Online BM    & 0.8580 & 0.8809 & 0.8949 \\
\end{tabular}
\end{table*}

In Table \ref{tb:empirical_coverage}, we compare the empirical coverage rates of confidence intervals constructed using de-biased and Online BM estimators averaged over all coordinates of $x^*$. The nominal coverage level is $95\%$. For linear and expectile regression, results are recorded at $n=60000$ for $d=5$, $n = 200000$ for $d = 20$, and $n = 500000$ for $d = 50$. For logistic regression, due to its highly non-linear and non-strongly convex nature, we set the batch size constant $C=4$ and record the coverage at $n=20000$ for $d=5$ to ensure convergence, while keeping the sample sizes for $d=20$ and $d=50$ the same as above. As shown in Table \ref{tb:empirical_coverage}, the de-biased estimator surpasses the Online BM method by achieving empirical coverage rates closer to the nominal $95\%$ level across all scenarios, making it a preferable option for practitioners conducting statistical inference on model parameters.

\end{document}